\documentclass{article}
\usepackage{caption}
\usepackage{subcaption}
\usepackage{amssymb,amsmath,graphicx,float}                                                                             
\usepackage{amsthm}                                      
\usepackage{mathtools}
\usepackage{hyperref}
\usepackage{color}             
\usepackage{cite}                                                                                                 
\usepackage{tabularx}
\usepackage{algorithm, algpseudocode}

\newtheorem{definition}{Definition}

\newtheorem{problem}{Problem}
\newtheorem{proposition}{Proposition}
\newtheorem{remark}{Remark}
\newcommand{\argmin}{\mathrm{arg}\min}

\newcommand{\norm}[1]{\left\lVert#1\right\rVert}
   
\usepackage{times}
\usepackage{tikz}
\usetikzlibrary{calc,decorations.pathmorphing,shapes}                                                                  

\newcounter{sarrow}
\newcommand\xrsquigarrow[1]{
	\stepcounter{sarrow}
	\mathrel{\begin{tikzpicture}[baseline= {( $(current bounding box.south) + (0,-0.5ex)$ )}]
		\node[inner sep=.5ex] (\thesarrow) {$\scriptstyle #1$};
		\path[draw,<-,decorate,
		decoration={zigzag,amplitude=0.7pt,segment length=1.2mm,pre=lineto,pre length=4pt}] 
		(\thesarrow.south east) -- (\thesarrow.south west);
		\end{tikzpicture}}
}

\title{Transfer of Temporal Logic Formulas in Reinforcement Learning}

\author{Zhe~Xu\thanks{Zhe~Xu is with the Oden Institute
		for Computational Engineering and Sciences, University of Texas,
		Austin, Austin, TX 78712, Ufuk Topcu is with the Department
		of Aerospace Engineering and Engineering Mechanics, and the Oden Institute
		for Computational Engineering and Sciences, University of Texas,
		Austin, Austin, TX 78712, e-mail: zhexu@utexas.edu, utopcu@utexas.edu.}~ and Ufuk Topcu
}
  
\usepackage[margin=1.5in]{geometry}

\date{}

\begin{document}                                                                                                                                     
	 
	\maketitle                                                                                 
	
	\begin{abstract}                                                                                                                        
		Transferring high-level knowledge from a \textit{source task} to a \textit{target task} is an effective way to expedite reinforcement learning (RL). For example, propositional logic and first-order logic have been used as representations of such knowledge. We study the transfer of knowledge between tasks in which the timing of the events matters. We call such tasks  \textit{temporal tasks}. We concretize similarity between temporal tasks through a notion of \textit{logical transferability}, and develop a transfer learning approach between different yet \textit{similar} temporal tasks. We first propose an inference technique to extract \textit{metric
		interval temporal logic} (MITL) formulas in \textit{sequential disjunctive normal form} from labeled trajectories collected in RL of the two tasks. If logical transferability is identified through this inference, we construct a timed automaton for each \textit{sequential conjunctive subformula} of the inferred MITL formulas from both tasks. We perform RL on the \textit{extended state} which includes the locations and clock valuations of the timed automata for the source task. We then establish mappings between the corresponding components (clocks, locations, etc.) of the timed automata from the two tasks, and transfer the \textit{extended Q-functions} based on the established mappings. Finally, we perform RL on the \textit{extended state} for the target task, starting with the transferred extended Q-functions. Our results in \color{black}two case studies \color{black} show, depending on how similar the source task and the target task are, that the sampling efficiency for the target task can be improved by up to one order of magnitude by performing RL in the extended state space, and further improved by up to another order of magnitude using the transferred extended Q-functions.                                     
	\end{abstract}                                                                                                                  
	                                                                                                                                              
	\section{Introduction}                                                         
	\label{sec_intro}
	Reinforcement learning (RL) has been successful in numerous applications. In practice though, it often requires extensive exploration of the environment to achieve satisfactory performance, especially for complex tasks with sparse rewards \cite{Wang2017}. 
	
	The sampling efficiency and performance of RL can be improved if some high-level knowledge can be incorporated in the learning process \cite{tor-etal-ccai18}. Such knowledge can be also transferred from a \textit{source task} to a \textit{target task} if these tasks are \textit{logically similar} \cite{Taylor_ICML2007}. For example, propositional logic and first-order logic have been used as representations of knowledge in the form of \textit{logical structures} for transfer learning \cite{mihalkova_aaai07}. They showed that incorporating such logical similarities can expedite RL for the target task \cite{Torrey2008RuleEF}.                                                                      
	
	The transfer of high-level knowledge can be also applied to tasks where the timing of the events matters. We call such tasks as \textit{temporal tasks}. Consider the gridworld example in Fig.\ref{map_intro}. In the source task, the robot should first reach a green region $G^{\textrm{S}}$ and stay there for at least 4 time units, then reach another yellow region $Y^{\textrm{S}}$ within 40 time units. 
    In the target task, the robot should first reach a green region $G^{\textrm{T}}$ and stay there for at least 5 time units, then reach another yellow region $Y^{\textrm{T}}$ within 40 time units. In both tasks, the green and yellow regions are a priori unknown to the robot. After 40 time units, the robot obtains a reward of 100 if it has completed the task and obtains a reward of -10 otherwise. It is intuitive that the two tasks are similar at a high level despite the differences in the specific regions in the workspace and timing requirements. 
         	
    \begin{figure}[th]
    	\centering                              
    	\includegraphics[width=10cm]{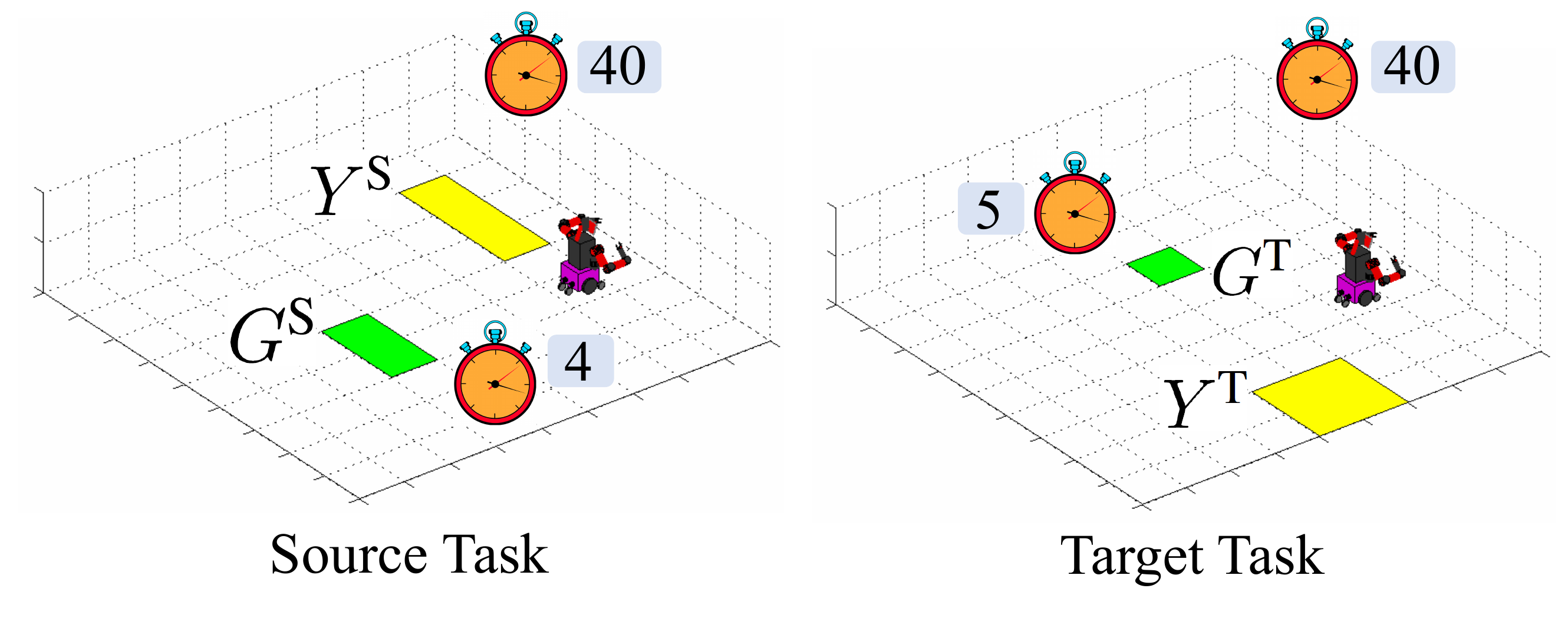}\caption{An illustrative example where the source task and the target task are \textit{logically similar}.}                                            
    	\label{map_intro} 
    \end{figure}

    Transfer learning between temporal tasks is complicated due to the following factors: (a) No formally defined criterion exists for logical similarities between temporal tasks. (b) Logical similarities are often implicit and need to be identified from data. (c) There is no known automated mechanism to transfer the knowledge based on logical similarities.                                                                                

    In this paper, we propose a transfer learning approach for temporal tasks in two levels: transfer of logical structures and transfer of low-level implementations. For ease of presentation, we focus on Q-learning \cite{Watkins1992}, while the general methodology applies readily to other forms of RL. 
    
    In the first level, we represent the high-level knowledge in temporal logic \cite{Pnueli}, which has been used in many applications in robotics and artificial intelligence \cite{Kress2011,SonThanh_MTL}. Specifically, we use a fragment of metric interval temporal logic (MITL) with bounded time intervals. We transfer such knowledge from a source task to a target task based on the hypothesis of \textit{logical transferability} (this notion will be formalized in Section \ref{Sec_Logical_Transferability}) between the two  tasks.                            
                                                               
\color{black}
    To identify logical transferability, we develop an inference technique that extracts \textit{informative} MITL formulas (this notion will be formalized in Section \ref{Sec_info}) in \textit{sequential disjunctive normal form}. These formulas effectively classify the labeled trajectories collected in RL of the two tasks. Referring back to the example shown in Fig.\ref{map_intro}, the regions corresponding to the \textit{atomic predicates} of the inferred MITL formulas are shown in Fig. \ref{map_intro_infer} (see Section \ref{case_I} for details). It can be seen that the inferred $\bar{G}^{\textrm{S}}$, $\bar{Y}^{\textrm{S}}$ and $\bar{G}^{\textrm{T}}$ are exactly the same as $G^{\textrm{S}}$, $Y^{\textrm{S}}$ and $G^{\textrm{T}}$, the inferred $\bar{Y}^{\textrm{T}}$ is different but close to $Y^{\textrm{T}}$.                                                       
 
   \begin{figure}[th]          
	   \centering                       
	   \color{black}       
        \includegraphics[width=10cm]{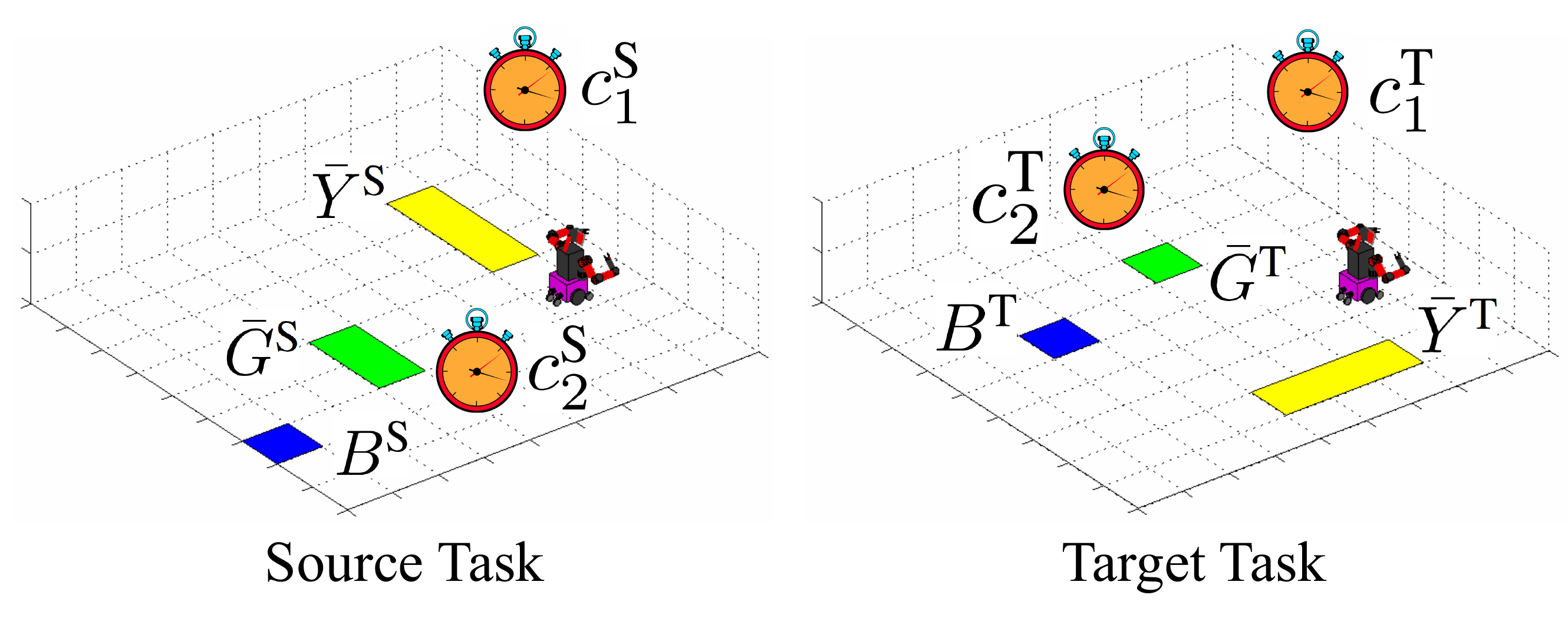}\caption{Inferred regions in the illustrative example.}   
	\label{map_intro_infer} 
    \end{figure}

   If the inference process indeed identifies logical transferability and extracts the associated MITL formulas, we construct a timed automaton for each \textit{sequential conjunctive subformula} of the inferred MITL formulas from the source task and the target task. For example, for the target task, a clock $c^{\textrm{T}}_1$ starts from the beginning for recording the time to reach the green and yellow regions, and a clock $c^{\textrm{T}}_2$ only starts when the robot reaches the green regions to record whether it stays there for 5 time units. The \textit{locations} of the timed automaton mark the stages of the task completion, as the location changes when the robot reaches $\bar{G}^{\textrm{T}}$, and also changes when the robot stays in $\bar{G}^{\textrm{T}}$ for 5 time units. We combine the locations and clock valuations of the timed automaton with the state of the robot to form an \textit{extended state}, and perform RL in the extended state space.                                                                        
 
   In the second level, we transfer the \textit{extended Q-functions} (i.e., Q-function on the extended states) from the source task to the target task. We first perform RL in the extended state space
   for the source task. Next, we establish mappings between the corresponding components (clocks, locations, etc.) of the timed automata from the two tasks based on the identified logical transferability. As in the example, we establish mappings between the regions $\bar{G}^{\textrm{S}}$ and $\bar{G}^{\textrm{T}}$, and between the regions $\bar{Y}^{\textrm{S}}$ and $\bar{Y}^{\textrm{T}}$. Similar mappings are established between the clocks $c^{\textrm{S}}_1$ and $c^{\textrm{T}}_1$, and between the clocks $c^{\textrm{S}}_2$ and $c^{\textrm{T}}_2$. Then, we transfer the extended Q-functions based on these mappings. For example, before the green regions are reached, $B^{\textrm{S}}$ is the most similar to $B^{\textrm{T}}$ in relative positions with respect to the centers of $\bar{G}^{\textrm{S}}$ and $\bar{G}^{\textrm{T}}$, respectively. Therefore, the extended Q-function with certain clock valuation and the state in $B^{\textrm{T}}$ in the target task is transferred from the extended Q-function with the most similar clock valuation and the state in $B^{\textrm{S}}$ in the source task. Finally, we perform Q-learning in the extended state space starting with the transferred extended Q-functions.           
   \color{black}
   
   The implementation of the proposed approach in two case studies show, in both levels, that the sampling efficiency is significantly improved for RL of the target task.                                                                                                                                                                     
	
   \noindent\textbf{Related Work.} Our work is closely related to the work on RL with temporal logic specifications \cite{Aksaray2016,Li2017,Icarte2018,Fu2014ProbablyAC,Min2017,Alshiekh2018SafeRL}.  The current results mostly rely on the assumption that the high-level knowledge (i.e., temporal logic specifications) are given, \color{black}while in reality they are often implicit and need to be inferred from data.\color{black}
   
   The methods for inferring temporal logic formulas from data can be found in \cite{Hoxha2017,Kong2017TAC,Bombara2016,Neider,zhe_advisory,zhe2016,VazquezChanlatte2018LearningTS,zhe_info,NIPS2018Shah}. \color{black}The inference method used in this paper is inspired from \cite{Bombara2016} and \cite{zhe_info}.\color{black}
   
  \color{black}While there has been no existing work on RL-based transfer learning utilizing similarity between temporal logic formulas, \color{black}the related work on transferring first-order logical structures or rules for expediting RL can be found in \cite{Taylor_ICML2007,Torrey_MLN2010,Torrey2008RuleEF}, \color{black}and the related work on transferring logical relations for action-model acquisition can be found in \cite{zhuo2014}.  \color{black}                       
	
	\section{Preliminaries}		  
\subsection{Metric Interval Temporal Logic}
\label{sec_MITL}
Let $\mathbb{B}=\{\top, \bot\}$ (tautology and contradiction, respectively) be
the Boolean domain and $\mathbb{T}=\{0, 1, 2, \dots\}$ be a discrete set of time indices. The underlying
system is modeled by a \textit{Markov decision process} (MDP) $\mathcal{M}= (S, A, P)$, where the state space $S$ and action set $A$ are finite, $P: S\times A\times S\rightarrow[0,1]$ is a transition probability
distribution. A trajectory $s_{0:L}=s_0s_1\cdots s_L$ describing an evolution of the MDP $\mathcal{M}$ is a function from $\mathbb{T}$ to                    
$S$. Let $AP$ be a set of \textit{atomic predicates}. 

The syntax of the MITL$_{f}$ fragment of time-bounded MITL formulas is defined recursively as follows\footnote{Although other temporal operators such as \textquotedblleft Until \textquotedblright ($\mathcal{U}$) may also appear in the full syntax of MITL, they are omitted from the syntax here as they can be hard to interpret and are not often used for the inference of temporal logic formulas \cite{Kong2017TAC}.}:                                       
\[
\phi:=\top\mid \rho\mid\lnot\phi\mid\phi_{1}\wedge\phi_{2}\mid\phi_{1}\vee
\phi_{2}\mid \Diamond_{I}\phi\mid \Box_{I}\phi,                                                                                            
\]
where $\rho\in AP$ is an atomic                      
predicate; $\lnot$ (negation), $\wedge$ (conjunction), $\vee$ (disjunction)                          
are Boolean connectives; $\Diamond$ (eventually) and $\Box$ (always) are temporal operators; and $I$ is a bounded interval of
the form $I=[i_{1},i_{2}]$ $(i_{1}<i_{2}, i_{1},i_{2} \in \mathbb{T})$.  \color{black}For example, the MITL$_{f}$ formula $\Box_{[2, 5]}(x>3)$ reads as ``$x$ is always greater than 3 during the time interval [2, 5]''.\color{black}                         

A \textit{timed word} generated by a trajectory $s_{0:L}$ is defined as   
a sequence $(\mathcal{L}(s_{t_1}),t_1),\dots, (\mathcal{L}(s_{t_m}),t_m)$, where $\mathcal{L}: S\rightarrow2^{\mathcal{AP}}$ is a labeling function assigning to each state $s\in S$ a subset of atomic predicates in $\mathcal{AP}$ that hold true at state $s$, $t_1=0$, $t_m=L$, $t_{k-1}<t_k$ ($k\in[2,m]$) and for $k\in[1,m-1]$, $t_{k+1}$ is the largest time index such that $\mathcal{L}(s_{t})=\mathcal{L}(s_{t_{k}})$ for all $t\in[t_{k},t_{k+1})$. The satisfaction of an MITL$_{f}$ formula by timed words as Boolean
semantics can be found in \cite{Alur_Punctuality}. We
say that a trajectory $s_{0:L}$ satisfies an MITL$_{f}$ formula $\phi$, denoted as $s_{0:L}\models\phi$, if and only if the timed word generated by $s_{0:L}$ satisfies $\phi$. As the time intervals $I$ in MITL$_{f}$ formulas are bounded intervals, MITL$_{f}$ formulas can be satisfied and violated by trajectories of finite lengths.                   	

\subsection{Timed Automaton}
Let $C$ be a finite set of clock variables. The set $\mathcal{C}_C$ of clock constraints is defined by \cite{Ouaknine2005}
\[
\varphi_C:=\top\mid c\bowtie k \mid\varphi_1\wedge\varphi_2,
\]
where $k\in\mathbb{N}$, $c\in C$ and $\bowtie\in\{<,\le,>,\ge\}$.

\begin{definition}\cite{Alur94}                                                  
	\label{TA}
	A timed automaton is a tuple $\mathcal{A}=(\Sigma,\mathcal{Q},q^0, C, \mathcal{F},\Delta)$, where 
	$\Sigma$ is a finite alphabet of input symbols, $\mathcal{Q}$ is a set of locations, $q^0\in\mathcal{Q}$ is the initial location, $C$ is a finite set of clocks, $\mathcal{F}\subset\mathcal{Q}$ is a set of accepting locations, $\Delta\subset\mathcal{Q}\times\Sigma\times\mathcal{Q}\times\mathcal{C}_C\times2^{C}$ is the transition function, $e=(q, \sigma, q', \varphi_C, r_C)\in \Delta$ represents a transition from $q$ to $q'$ labeled by $\sigma$, provided the precondition $\varphi_C$ on the clocks is met, $r_C$ is the set of clocks that are reset to zero.                                                                                             
\end{definition}                                                       

\begin{remark}
	We focus on timed automata with discrete time, which are also called \textit{tick automata} in \cite{Gruber_tick}. 
\end{remark}

A timed automaton $\mathcal{A}$ is \textit{deterministic} if and only if for each location and input symbol there is at most one transition to the next location. We denote by $v=(v_1,\dots,v_{\vert C\vert})\in V\subset\mathbb{T}^{\vert C\vert}$ the \textit{clock valuation} of $\mathcal{A}$ (we denote by $\vert C\vert$ the cardinality of $C$), where $v_k\in V_k\subset\mathbb{T}$ is the value of clock $c_k\in C$. 
For a timed word $\nu=(\sigma_0,t_0), (\sigma_1,t_1), \dots, (\sigma_m,t_m)$ (where $t_0<t_1<\dots<t_m$, $\sigma_k\in\Sigma$ for $k\in[0,m]$) and writing $d_k:=t_{k+1}-t_k$, a \textit{run} of $\mathcal{A}$ on $\nu$ is defined as\\
$(q^0, v^0)\xrightarrow{\sigma_0}(q^1, v^1)\xrsquigarrow{d_0}(q^2, v^2)\xrightarrow{\sigma_1}(q^3,  v^3)\xrsquigarrow{d_1}\dots\xrsquigarrow{d_{m-1}}(q^{2m}, v^{2m})\xrightarrow{\sigma_m}(q^{2m+1}, v^{2m+1})$, \\
where the \textit{flow-step} relation is defined by $(q, v) \xrightarrow{d} (q, v+d)$ where $d\in\mathbb{R}_{>0}$; the \textit{edge-step} relation is defined by $(q, v) \xrsquigarrow{\sigma} (q', v')$ if and only if there is an edge $(q, \sigma, q', \varphi_{C}, r_{C})\in \Delta$ such that $\sigma\in\Sigma$, $v$ satisfies $\varphi_{C}$, $v'_k=0$ for all $c_k\in r_{C}$ and $v'_k=v_k$ for all $c_k\not\in r_{C}$. A finite run is \textit{accepting} if the last location in the run belongs to $\mathcal{F}$. A timed word $\nu$ is accepted by $\mathcal{A}$ if there is some accepting run of $\mathcal{A}$ on $\nu$.                                                                                                                                                        
	
\section{Information-Guided Inference of Temporal Logic Formulas}                                                      
\label{Sec_info}
We now introduce the \textit{information gain} provided by an MITL$_{f}$ formula, the problem formulation and the algorithm to extract MITL$_{f}$ formulas from labeled trajectories.     

\subsection{Information Gain of MITL$_{f}$ Formulas} 
We denote by $\mathcal{B}_L$ the set of all possible trajectories with length $L$ generated by the MDP $\mathcal{M}$, and use $\mathcal{G}_L: \mathcal{B}_L\rightarrow[0,1]$ to denote a prior probability distribution (e.g., uniform distribution) over $\mathcal{B}_L$. We use $\mathbb{P}_{\mathcal{B}_L,\phi}$ to denote the probability of a trajectory $s_{0:L}$ satisfying $\phi$ in $\mathcal{B}_L$ based on $\mathcal{G}_L$.                                        

\begin{definition}
	Given a prior probability distribution $\mathcal{G}_L$ and an MITL$_{f}$ formula $\phi$ such that $\mathbb{P}_{\mathcal{B}_L,\phi}>0$, we define $\mathcal{\bar{G}}^{\phi}_L: \mathcal{B}_L\rightarrow[0,1]$ as the posterior probability distribution, given that $\phi$ evaluates to true, which is expressed as
	\[                           
	\begin{split}
	\mathcal{\bar{G}}^{\phi}_L(s_{0:L}):=
	\begin{cases}
	\frac{\mathcal{G}_L(s_{0:L})}{\mathbb{P}_{\mathcal{B}_L,\phi}}, ~~~~~~~~\mbox{if}~s_{0:L}\models\phi,\\ 
	0,~~~~~~~~~~~~~~~~~~~~~\mbox{otherwise}.
	\end{cases} 
	\end{split}
	\]
	\label{know}
\end{definition}

The expression of $\mathcal{\bar{G}}^{\phi}_L$ can be derived using Bayes' theorem. We use the fact that the probability of $\phi$ evaluating to true given $s_{0:L}$ is 1, if $s_{0:L}$ satisfies $\phi$; and it is 0 otherwise.	

\begin{definition}     
	When the prior probability distribution $\mathcal{G}_L$ is updated to the posterior probability distribution $\mathcal{\bar{G}}^{\phi}_L$, we define the information gain as                         
	\begin{center}   
		$\mathcal{I}(\mathcal{G}_L,\mathcal{\bar{G}}^{\phi}_L):=D_{\rm{KL}}(\mathcal{\bar{G}}^{\phi}_L\vert\vert\mathcal{G}_L)/L,$
	\end{center}   
	where $D_{\rm{KL}}(\mathcal{\bar{G}}^{\phi}_L\vert\vert\mathcal{G}_L)$
	is the Kullback-Leibler divergence from $\mathcal{G}_L$ to $\mathcal{\bar{G}}^{\phi}_L$.                                              
	\label{KL}
\end{definition}  

\begin{proposition} 
	\label{exp} 
	For an MITL$_{f}$ formula $\phi$, if $\mathbb{P}_{\mathcal{B}_L,\phi}>0$, then         
	\begin{center}                 
		$\mathcal{I}(\mathcal{G}_L,\mathcal{\bar{G}}^{\phi}_L) 
		=-\log{\mathbb{P}_{\mathcal{B}_L,\phi}}/L.$
	\end{center}      
\end{proposition}  
\begin{proof} 
	Straightforward from Definitions \ref{know} and \ref{KL}.
\end{proof} 
If $\phi=\top$, then $\mathbb{P}_{\mathcal{B}_L,\phi}=1$ and $\mathcal{I}(\mathcal{G}_L,\mathcal{\bar{G}}^{\phi}_L)=0$, i.e., tautologies provide no information gain. For completeness, we also define that the information gain $\mathcal{I}(\mathcal{G}_L,$ $\mathcal{\bar{G}}^{\phi}_L)=0$ if $\mathbb{P}_{\mathcal{B}_L,\phi}=0$. So if $\phi=\bot$, then $\mathbb{P}_{\mathcal{B}_L,\phi}=0$ and $\mathcal{I}(\mathcal{G}_L,\mathcal{\bar{G}}^{\phi}_L)=0$, i.e., contradictions provide no information gain.                                 

For two MITL$_{f}$ formulas $\phi_1$ and $\phi_2$, we say $\phi_1$ is more \textit{informative} than $\phi_2$ with respect to the prior probability distribution $\mathcal{G}_L$ if $\mathcal{I}(\mathcal{G}_L,\mathcal{\bar{G}}^{\phi_1}_L)>\mathcal{I}(\mathcal{G}_L,\mathcal{\bar{G}}^{\phi_2}_L)$.

Based on Proposition \ref{exp}, the computation of the information gain requires the computation of $\mathbb{P}_{\mathcal{B}_L,\phi}$. We point the reader to \cite{zhe_info} for a recursive method to compute $\mathbb{P}_{\mathcal{B}_L,\phi}$.                                  
                                                         
\subsection{Problem Formulation} 
We now provide some related definitions for formulating the inference problem. Let a set $\mathcal{P}$ of \textit{primitive structures} \cite{Bombara2016} used in the rest of the paper be
\begin{align}
\mathcal{P}:=\{\Diamond_{I}\rho, \Box_{I}\rho, \Diamond_{I}\Box_{I'}\rho, \Box_{I}\Diamond_{I'}\rho\},          
\label{template}
\end{align}                                       
where $I=[i_{1},i_{2}]$ $(i_{1}<i_{2}, i_{1},i_{2}\in \mathbb{T})$, $I'=[0,i_{2}]$ $(i_{2}>0,$ $i_{2} \in \mathbb{T})$, and $\rho$ is an atomic predicate. We call an MITL$_{f}$ formula $\phi$ a \textit{primitive MITL$_{f}$ formula} if $\phi$ follows one of the primitive structures in $\mathcal{P}$ or the negation of such a structure.

\begin{definition}
	For an MITL$_{f}$ formula $\phi$, we define the start-effect time $t_{\rm{s}}(\phi)$ and end-effect time $t_{\rm{e}}(\phi)$ recursively as                                                            
	\begin{align}\nonumber
	\begin{split}  
	t_{\rm{s}}(\color{black}\rho\color{black})=& t_{\rm{e}}(\color{black}\rho\color{black})= 0,
	t_{\rm{s}}(\lnot\phi) = t_{\rm{s}}(\phi), t_{\rm{e}}(\lnot\phi) = t_{\rm{e}}(\phi), \\
	t_{\rm{s}}(\phi_1\wedge\phi_2) =& \min\{t_{\rm{s}}(\phi_1), t_{\rm{s}}(\phi_2)\}, \\t_{\rm{e}}(\phi_1\wedge\phi_2) =& \max\{t_{\rm{e}}(\phi_1), t_{\rm{e}}(\phi_2)\},\\
	t_{\rm{s}}(\Diamond_{[t_1, t_2]}\phi) =& t_{\rm{s}}(\phi)+t_1, ~~~~~~~t_{\rm{e}}(\Diamond_{[t_1, t_2]}\phi) = t_{\rm{e}}(\phi)+t_2,\\	
	t_{\rm{s}}(\Box_{[t_1, t_2]}\phi) =& t_{\rm{s}}(\phi)+t_1, ~~~~~~~t_{\rm{e}}(\Box_{[t_1, t_2]}\phi) = t_{\rm{e}}(\phi)+t_2.
	\end{split}
	\end{align}                               		
	\label{effect} 
\end{definition}            

\begin{definition}
	\label{MITL_seq}
	An MITL$_{f}$ formula $\phi$ is in disjunctive normal form if $\phi$ is expressed in the form of $(\phi^1_1\wedge\dots\wedge\phi^{n_1}_1)\vee\dots\vee(\phi^1_{m}\wedge\dots\wedge\phi^{n_m}_m)$, where each $\phi^j_i$ is a primitive MITL$_{f}$ formula (also called primitive subformula of $\phi$). If, for any $i\in[1,m]$ and for all $j, k\in[1, n_i]$ such that $j<k$, it holds that $t_{\rm{e}}(\phi^j_i)<t_{\rm{s}}(\phi^k_i)$, then we say $\phi$ is in sequential disjunctive normal form (SDNF) and we call each $\phi_i:=\phi^1_{i}\wedge\dots\wedge\phi^{n_i}_i$ a \textit{sequential conjunctive subformula}.
\end{definition} 

In the following, we consider MITL$_{f}$ formulas only in the SDNF for reasons that will become clear in Section \ref{Sec_Transfer_Learning}. We define the \textit{size} of an MITL$_{f}$ formula $\phi$ in the SDNF, denoted as $\varrho(\phi)$, as the number of primitive MITL$_{f}$ formulas in $\phi$. 

Suppose that we are given a set $\mathcal{S}_L$ $=\{(s^k_{0:L},l_k)\}^{N_{\mathcal{S}_L}}_{k=1}$ of labeled trajectories, where $l_k=1$ and $l_k=-1$ represent desired and undesired behaviors, respectively. We define the \textit{satisfaction signature} $g_{\phi}(s^k_{0:L})$ of a trajectory $s^k_{0:L}$ as follows: $g_{\phi}(s^k_{0:L})=1$, if $s^k_{0:L}$ satisfies $\phi$; and $g_{\phi}(s^k_{0:L})=-1$, if $s^k_{0:L}$ does not satisfy $\phi$. Note that here we assume that $L$ is sufficiently large, thus $s^k_{0:L}$
either satisfies or violates $\phi$. A labeled trajectory $(s^k_{0:L}, l_k)$ is \textit{misclassified} by $\phi$ if $g_{\phi}(s^k_{0:L})\neq l_k$. We use $CR(S_L, \phi)=\vert\{(s^k_{0:L},l_k)\in\mathcal{S}_L: g_{\phi}(s^k_{0:L})=l_k\}\vert/\vert\mathcal{S}_L\vert$ to denote the classification rate of $\phi$ in $\mathcal{S}_L$.                                                           

\begin{problem}
	Given a set $\mathcal{S}_L=\{(s^k_{0:L},$ $l_k)\}^{N_{\mathcal{S}_L}}_{k=1}$ of labeled trajectories, a prior probability distribution $\mathcal{G}_L$, real constant $\zeta\in(0,1]$ and integer constant $\varrho_{\rm{th}}\in(0, \infty)$, construct an MITL$_{f}$ formula $\phi$ in the SDNF that maximizes $\mathcal{I}(\mathcal{G}_L,\mathcal{\bar{G}}^{\phi}_L)$ while satisfying
	\begin{itemize}
		\item the classification constraint $CR(\mathcal{S}_L,\phi)\ge\zeta$ and
		\item the size constraint
		$\varrho(\phi)\le\varrho_{\rm{th}}$. 
	\end{itemize}      
	\label{problem1}                                        
\end{problem}  

Intuitively, as there could be many MITL$_{f}$ formulas that satisfy the classification constraint and the size constraint, we intend to obtain the most informative one to be utilized and transferred as features of desired behaviors.	

We call an MITL$_{f}$ formula that satisfies both the two constraints of Problem \ref{problem1} a \textit{satisfying formula} for $\mathcal{S}_L$.

\subsection{Solution Based on Decision Tree} 	
We propose an inference technique, which is inspired by \cite{Bombara2016} and \cite{zhe_info}, in order to solve Problem \ref{problem1}. The technique consists of two steps. In the first step, we construct a decision tree where each non-leaf node is associated with a primitive MITL$_{f}$ formula (see the formulas inside the circles in Fig. \ref{tree}). In the second step, we convert the constructed decision tree to an MITL$_{f}$ formula in the SDNF.                                                                                                                            

\begin{figure}[th]
	\centering
	\includegraphics[width=10cm]{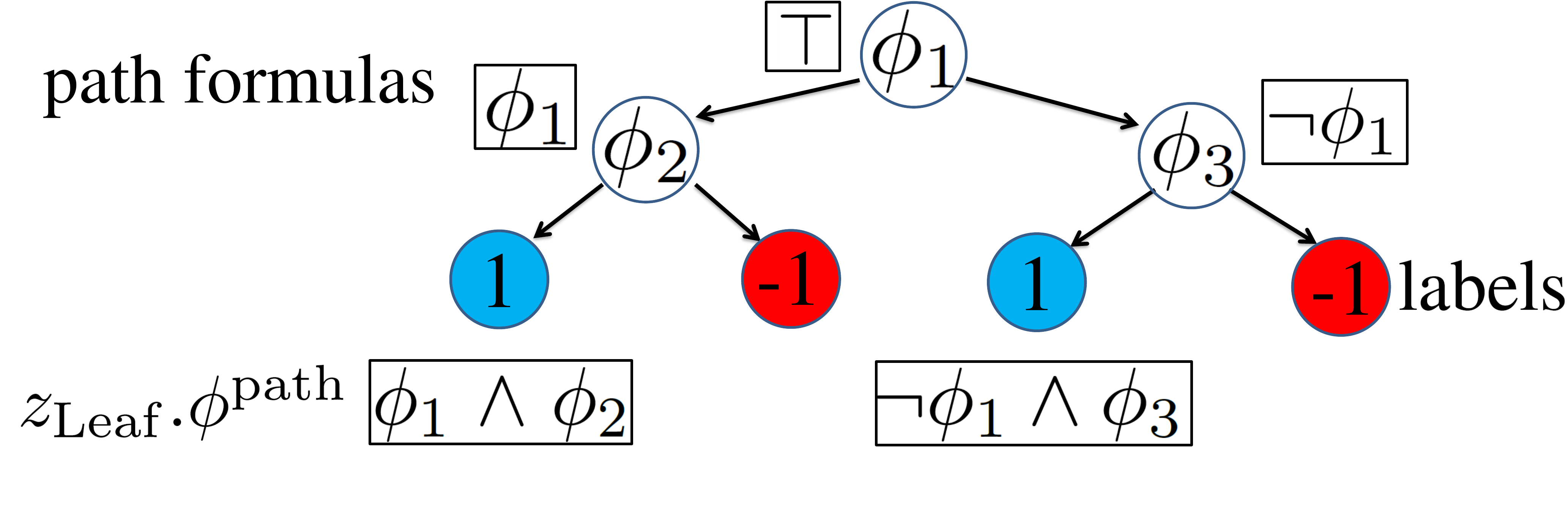}\caption{Illustration of a decision tree which can be converted to an MITL$_{f}$ formula in the SDNF.} 
	\label{tree}                                                                                                      
\end{figure}

In Algorithm \ref{algTree}, we construct the decision tree by recursively calling the $MITLtree$ procedure from the root node to each leaf node. There are three inputs to the $MITLtree$ procedure: (1) a set $\mathcal{S}$ of labeled trajectories assigned to the current node; (2) a formula $\phi^{\rm{path}}$ to reach the current node (also called the \textit{path formula}, see the formulas inside the rectangles in Fig. \ref{tree}); and (3) the depth $h$ of the current node. The set $\mathcal{S}$ assigned to the root node is initialized as $\mathcal{S}_L$ in Problem \ref{problem1}, $\phi^{\rm{path}}$ and $h$ are initialized as $\top$ and 0, respectively. 

For each node, we set a criterion $stop(\phi^{\rm{path}}, h, \mathcal{S})$ to determine whether it is a leaf node (Line \ref{stop}). Each leaf node is associated with label 1 or -1, depending on whether more than 50\% of the labeled trajectories assigned to that node are with label 1 or not.       

\begin{algorithm}[t]
	\caption{Information-Guided MITL$_{f}$ Inference.}
	\label{algTree}
	\begin{algorithmic}[1]
		\State \textbf{procedure} $MITLtree(\mathcal{S}=\{(s^k_{0:L}, l_k)\}^{N_{\mathcal{S}}}_{k=1}, \phi^{\rm{path}}, h)$
		\State \textbf{if} $stop(\phi^{\rm{path}}, h, \mathcal{S})$ \textbf{then} \label{stop}
		\State ~~Create node $z_{\rm{Leaf}}$ as a leaf node 
		\State ~~\textbf{if} node $z_{\rm{Leaf}}$ is associated with label 1 \label{leaf_s}	
		\State ~~~~$z_{\rm{Leaf}}.\phi^{\rm{path}}\gets\phi^{\rm{path}}$                                   
		\State ~~\textbf{end if} \label{leaf_e}	
		\State ~~return $z$ 	
		\State \textbf{end if} 
		\State Create node $z$ as a non-leaf node                  
		\State  Obtain $\Upsilon_z$ from (\ref{Theta_exp})                       
		\State $\displaystyle z.\phi\gets\arg\max_{\phi\in \mathcal{P},\theta\in\Upsilon_z}CR(\mathcal{S},\phi_{\theta})+\lambda\mathcal{I}(\mathcal{G}_L,\mathcal{\bar{G}}^{\phi_{\theta}}_L)$ \label{cost_func}
		\State $\{\mathcal{S}_{\top}, \mathcal{S}_{\bot}\}\leftarrow partition(\mathcal{S}, z.\phi)$ \label{part}
		\State $z.left\leftarrow MITLtree(\mathcal{S}_{\top}, \phi^{\rm{path}}\wedge z.\phi, h+1)$ \label{recur_s}	
		\State $z.right\leftarrow MITLtree(\mathcal{S}_{\bot}, \phi^{\rm{path}}\wedge\neg z.\phi, h+1)$ \label{recur_e}		
		\State return $z$
		\State\textbf{end procedure}                                                                                                      
	\end{algorithmic}
\end{algorithm}   

At each non-leaf node $z$, we construct a primitive MITL$_{f}$ formula $\phi_{\theta}$ parameterized by $\theta\in\Upsilon_z$, where 
\begin{align}
\begin{split}
\Upsilon_z:=&\{\theta| [t_{\rm{s}}(\phi_{\theta}),t_{\rm{e}}(\phi_{\theta})]\cap[t_{\rm{s}}(\phi'),t_{\rm{e}}(\phi')]=\emptyset \\
& \textrm{for each primitive subformula}~\phi'~\textrm{of}~\phi^{\rm{path}}\}. 
\label{Theta_exp}                                                                                  
\end{split}
\end{align}
\color{black}For example, for an MITL$_{f}$ formula $\phi_{\theta}=\Box_{[i_1, i_2]}(x>a)$, we have $\theta=[i_1, i_2, a]$. If $\phi^{\rm{path}}=\Diamond_{[1,15]}\Box_{[0,4]}(x>3)$, then $\Upsilon_z$ ensures that the start-effect time of $\phi_{\theta}$ is later than the end-effect time of $\phi^{\rm{path}}$ (which is 19). Essentially $\Upsilon_z$ guarantees that the primitive MITL$_{f}$ formula $\phi_{\theta}$ and primitive subformulas of $\phi^{\rm{path}}$ can be reordered to form a sequential conjunctive subformula (see Definition \ref{MITL_seq}). \color{black}

We use particle swarm optimization (PSO) \cite{PSO_Eberhart} to optimize $\theta$ for each primitive structure from $\mathcal{P}$ and compute a primitive MITL$_{f}$ formula $z.\phi=\phi^{\ast}_{\theta}$ which maximizes the objective function                
\begin{align}
\begin{split}
J(\mathcal{S},\phi_{\theta}):=& CR(\mathcal{S},\phi_{\theta})+\lambda\mathcal{I}(\mathcal{G}_L,\mathcal{\bar{G}}^{\phi_{\theta}}_L)
\end{split}
\label{obj}
\end{align}                                                                                           
in Line \ref{cost_func}, where $\lambda$ is a weighting factor.

With $z.\phi$, we partition the set $\mathcal{S}$ into $\mathcal{S}_{\top}$ and $\mathcal{S}_{\bot}$, where the trajectories in $\mathcal{S}_{\top}$ and $\mathcal{S}_{\bot}$ satisfy and violate $z.\phi$, respectively (Line \ref{part}). Then the procedure is called recursively to construct the left and right sub-trees for $\mathcal{S}_{\top}$ and $\mathcal{S}_{\bot}$, respectively (Lines \ref{recur_s}, \ref{recur_e}).

After the decision tree is constructed, for each leaf node associated with label 1, it is also associated with a path formula $z_{\rm{Leaf}}.\phi^{\rm{path}}$ (Line \ref{leaf_s} to Line \ref{leaf_e}). The path formula $z_{\rm{Leaf}}.\phi^{\rm{path}}$ is constructed recursively from the associated primitive 
MITL$_{f}$ formulas along the path from the root node to the parent of the leaf node (see Fig. \ref{tree}). We rearrange the primitive subformulas of each $z_{\rm{Leaf}}.\phi^{\rm{path}}$ in the order of increasing start-effect time to obtain a sequential conjunctive subformula. We then connect all the obtained sequential conjunctive subformulas with disjunctions. In this way, the obtained decision tree can be converted to an MITL$_{f}$ formula in the SDNF. As in the example shown in Fig. \ref{tree}, if $t_{\rm{s}}(\phi_1)<t_{\rm{s}}(\phi_2)$ and $t_{\rm{s}}(\lnot\phi_1)<t_{\rm{s}}(\phi_3)$, then the decision tree can be converted to $(\phi_1\wedge\phi_2)\vee(\lnot\phi_1\wedge\phi_3)$ in the SDNF. If $t_{\rm{s}}(\phi_2)<t_{\rm{s}}(\phi_1)$ and $t_{\rm{s}}(\phi_3)<t_{\rm{s}}(\lnot\phi_1)$, then the decision tree can be converted to  $(\phi_2\wedge\phi_1)\vee(\phi_3\wedge\lnot\phi_1)$ in the SDNF.                                                                                                                                                     

We set the criterion $stop(\phi^{\rm{path}}, h, \mathcal{S})$ as follows. If at least $\zeta$ (e.g., $95\%$) of the labeled trajectories assigned to the node are with the same label (Condition I) or the depth $h$ of the node reaches a set maximal depth $h_{\textrm{max}}$ (Condition II) or $\Upsilon_z$, as defined in (\ref{Theta_exp}), becomes the empty set (Condition III), then the node is a leaf node. If condition I holds for each leaf mode, then the obtained MITL$_{f}$ formula satisfies the classification constraint of Problem \ref{problem1}. If we set $h_{\textrm{max}}2^{h_{\textrm{max}}-1}\le\varrho_{\rm{th}}$, then the size constraint is guaranteed to be satisfied. 

\color{black}
The complexity of Algorithm \ref{problem1} for the average case can be determined through the Akra-Bazzi method as follows \cite{Bombara2016}:
$$
\Theta\Big(N_{\mathcal{S}}\cdot\big(1+\int_{1}^{N_{\mathcal{S}}}\frac{f(u)}{u^{2}}\mathrm{d}u\big)\Big), 
$$
where $f(N_{\mathcal{S}})$ is the complexity of the local PSO algorithm for $N_{\mathcal{S}}$ labeled trajectories, and $\Theta(\cdot)$ denotes the two-sided asymptotic notation for complexity bound.             
\color{black}

\section{Transfer Learning of Temporal Tasks Based on Logical Transferability}
\label{Sec_Transfer_Learning}
In this section, we first introduce the notion of \textit{logical transferability}. Then, we present the framework and algorithms for utilizing logical transferability for transfer learning. 
	
	\subsection{Logical Transferability} 
	\label{Sec_Logical_Transferability}
	To define logical transferability, we first define the \textit{structural transferability} between two MITL$_{f}$ formulas.

	For each primitive MITL$_{f}$ formula $\phi$, we use $O_T(\phi)$ to denote the temporal operator in $\phi$. For example, $O_T(\Diamond_{[5,8]}(x>3))=\Diamond$ (\color{black}eventually\color{black}) and $O_T(\Box_{[0, 8]}\Diamond_{[0,4]}(x<5))=\Box\Diamond$ (\color{black}always eventually\color{black}).
	
	\begin{definition}
		\label{MITL_equ}
		Two MITL$_{f}$ formulas (in the SDNF)
		\begin{align}\nonumber
		\phi=(\phi^1_1\wedge\dots\wedge\phi^{n_1}_1)\vee\dots\vee(\phi^1_{m}\wedge\dots\wedge\phi^{n_m}_m)
		\end{align} and
		\begin{align}\nonumber \hat{\phi}=(\hat{\phi}^1_1\wedge\dots\wedge\hat{\phi}^{\hat{n}_1}_1)\vee\dots\vee(\hat{\phi}^1_{\hat{m}}\wedge\dots\wedge\hat{\phi}^{\hat{n}_{\hat{m}}}_{\hat{m}})
		\end{align}
		are structurally equivalent, if and only if the followings hold:\\
		(1) $m=\hat{m}$ and, for every $i\in [1, m], n_i=\hat{n}_i$; and\\
		(2) For every $i\in[1, m]$ and every $j\in[1, n_i]$, $O_T(\phi^j_{i})=O_T(\hat{\phi}^j_{i})$.             
	\end{definition}                  
	 \color{black}
	\begin{definition}                    
		\label{MITL_transfer}
		For two MITL$_{f}$ formulas $\phi_1$ and $\phi_2$ in the SDNF, $\phi_2$ is structurally transferable from $\phi_1$ if and only if either of the following conditions holds:\\
		1) $\phi_1$ and $\phi_2$ are structurally equivalent;\\
		2) $\phi_2$ is in the form of $\phi^1_2\vee\dots\vee\phi^p_2$ ($p>1$), where each $\phi^k_2~(k=1,\dots,p)$ is structurally equivalent with $\phi_1$.                   
	\end{definition}
	\color{black}
	Suppose that we are given a source task $\mathcal{T}^{\textrm{S}}$ in the source environment $\mathcal{E}^{\textrm{S}}$ and a target task $\mathcal{T}^{\textrm{T}}$ in the target environment $\mathcal{E}^{\textrm{T}}$, with two sets $\mathcal{S}^{\textrm{S}}_L$ and $\mathcal{S}^{\textrm{T}}_L$ of labeled trajectories collected \color{black}during the initial episodes of RL (which we call the \textit{data collection phase}) in $\mathcal{E}^{\textrm{S}}$ and $\mathcal{E}^{\textrm{T}}$ respectively\color{black}. The trajectories are labeled based on a given task-related performance criterion. \color{black}To ensure the quality of inference, the data collection phase is chosen such that both $\mathcal{S}^{\textrm{S}}_L$ and $\mathcal{S}^{\textrm{T}}_L$ contain sufficient labeled trajectories with both label 1 and label -1.\color{black} We give the following definition for logical transferability.
	\begin{definition}                    
		\label{logic_transfer}
	$\mathcal{T}^{\textrm{T}}$ is \textit{logically transferable} from $\mathcal{T}^{\textrm{S}}$ based on $\mathcal{S}^{\textrm{S}}_L$, $\mathcal{S}^{\textrm{T}}_L$, $\zeta$ and $\varrho_{\rm{th}}$ (as defined in Problem \ref{problem1}), if and only if there exist \textit{satisfying} formulas $\phi^{\textrm{S}}$ for $\mathcal{S}^{\textrm{S}}_L$ and $\phi^{\textrm{T}}$ for $\mathcal{S}^{\textrm{T}}_L$ such that $\phi^{\textrm{T}}$ is \textit{structurally \color{black}transferable\color{black}} from $\phi^{\textrm{S}}$.
	\end{definition}
    \color{black}
     As in the introductory example in Section \ref{sec_intro}, the inferred MITL$_{f}$ formulas $\phi^{\textrm{S}}=\Diamond_{[1,15]}\Box_{[0,4]}\bar{G}^{\textrm{S}}\wedge \Diamond_{[21,39]}\bar{Y}^{\textrm{S}}$ and  
    $\phi^{\textrm{T}}=\Diamond_{[5,18]}\Box_{[0,5]}\bar{G}^{\textrm{T}}\wedge \Diamond_{[24,39]}\bar{Y}^{\textrm{T}}$ are satisfying formulas for $\mathcal{S}^{\textrm{S}}_L$ and $\mathcal{S}^{\textrm{T}}_L$, respectively (see Section \ref{case_I} for details, where we set $\zeta=0.95$ and $\varrho_{\rm{th}}=4$). According to Definition \ref{MITL_equ}, $\phi^{\textrm{T}}$ is structurally equivalent with $\phi^{\textrm{S}}$. Therefore, $\mathcal{T}^{\textrm{T}}$ is logically transferable from $\mathcal{T}^{\textrm{S}}$ based on $\mathcal{S}^{\textrm{S}}_L$, $\mathcal{S}^{\textrm{T}}_L$, $\zeta=0.95$ and $\varrho_{\rm{th}}=4$. \color{black}

	\color{black} In the following, we explain the proposed transfer learning approach based on logical transferability in two different levels. We provide a workflow diagram as a general overview of the proposed transfer learning approach, as shown in Fig. \ref{block}. \color{black}
	
	\begin{figure}[th]
		\centering
		\color{black}
		\includegraphics[width=12cm]{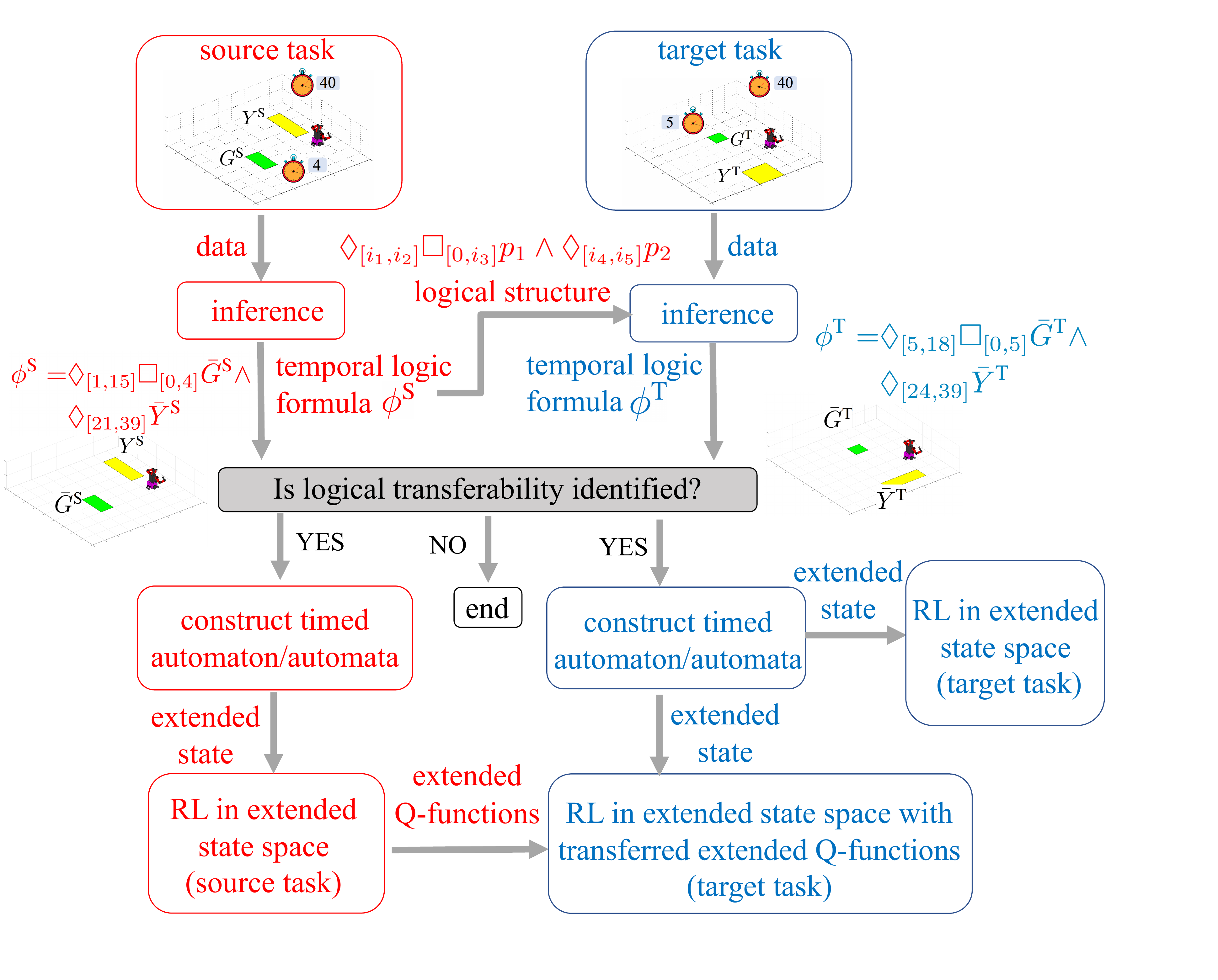}\caption{\color{black}Workflow diagram of the proposed transfer learning approach based on logical transferability.}                  
		\label{block}                                                                                                      
	\end{figure}

	\subsection{Transfer of Logical Structures Based on Hypothesis of Logical Transferability}                                             
	 \label{Sec_transferLogic}                                                            
	We first introduce the transfer of logical structures between temporal tasks. To this end, we pose the hypothesis that the target task is logically transferable from the source task. If logical transferability can be indeed identified, we perform RL for the target task utilizing the transferred logical structure. Specifically, we take the following three steps:

\noindent \textbf{Step 1}: \textit{Extracting MITL$_{f}$ formulas in the source task}.                                                  
From $\mathcal{S}_L^{\textrm{S}}$, we infer an MITL$_{f}$ formula $\phi^{\textrm{S}}$ using Algorithm \ref{algTree}. If $\phi^{\textrm{S}}$ is a satisfying formula for $\mathcal{S}_L^{\textrm{S}}$, we proceed to Step 2. 

\color{black} As in the introductory example, we obtain the satisfying formula $\phi^{\textrm{S}}=\Diamond_{[1,15]}\Box_{[0,4]}\bar{G}^{\textrm{S}}\wedge \Diamond_{[21,39]}\bar{Y}^{\textrm{S}}$ for $\mathcal{S}_L^{\textrm{S}}$. \color{black}\\

\noindent \textbf{Step 2}: \textit{Extracting MITL$_{f}$ formulas in the target task}.\\
From $\mathcal{S}_L^{\textrm{T}}$, we check if it is possible to infer a satisfying MITL$_{f}$ formula $\phi^{\textrm{T}}$ for $\mathcal{S}_L^{\textrm{T}}$ such that $\phi^{\textrm{T}}$ is structurally \color{black}transferable from \color{black} the inferred MITL$_{f}$ formula $\phi^{\textrm{S}}$ from the source task. \color{black}We start from inferring an MITL$_{f}$ formula that is structurally equivalent with $\phi^{\textrm{S}}$. This can be done by fixing the temporal operators (the same with those of $\phi^{\textrm{S}}$), then optimizing the parameters that appear in $\phi^{\textrm{T}}$  (through PSO) for maximizing the objective function in (\ref{obj}). If a satisfying MITL$_{f}$ formula is not found, we infer a MITL$_{f}$ formula $\phi^{\textrm{T}}$ in the form of $\phi^{\textrm{T}}=\phi^{\textrm{T}}_1\vee\phi^{\textrm{T}}_2$, where $\phi^{\textrm{T}}_1$ and $\phi^{\textrm{T}}_2$ are both structurally equivalent with $\phi^{\textrm{S}}$. In this way, we keep increasing the number of structurally equivalent formulas connected with disjunctions until a satisfying MITL$_{f}$ formula is found, or the size constraint is violated (i.e., $\varrho(\phi^{\textrm{T}})>\varrho_{\rm{th}}$). \color{black}If a satisfying MITL$_{f}$ formula is found, we proceed to Step 3; otherwise, logical transferability is not identified.  

 \color{black} As in the introductory example, we obtain the satisfying formula $\phi^{\textrm{T}}=\Diamond_{[5,18]}\Box_{[0,5]}\bar{G}^{\textrm{T}}\wedge \Diamond_{[24,39]}\bar{Y}^{\textrm{T}}$ for $\mathcal{S}_L^{\textrm{T}}$ and  $\phi^{\textrm{T}}$ is structurally equivalent with $\phi^{\textrm{S}}$, hence logical transferability is identified. \color{black}\\    

\noindent \textbf{Step 3}: \textit{Constructing timed automata and performing RL in the extended state space for the target task}.\\
For the satisfying formula $\phi^{\textrm{T}}=\phi^{\textrm{T}}_1\vee\dots\vee\phi^{\textrm{T}}_{m}$ in the SDNF, we can construct a deterministic timed automaton (DTA) $\mathcal{A}^{\phi^{\textrm{T}}_i}=\{2^{AP},\mathcal{Q}^{\phi^{\textrm{T}}_i},q^{0\phi^{\textrm{T}}_i}, C^{\phi^{\textrm{T}}_i}, \mathcal{F}^{\phi^{\textrm{T}}_i},\Delta^{\phi^{\textrm{T}}_i}\}$ \cite{Alur_Punctuality} that accepts precisely the timed words that satisfy each sequential conjunctive subformula $\phi^{\textrm{T}}_i$. 

We perform RL in the extended state space $X^{\textrm{T}}=\bigcup_{i} X^{\textrm{T}}_{i}$, where each $X^{\textrm{T}}_i=S^{\textrm{T}}\times\mathcal{Q}^{\phi^{\textrm{T}}_i}\times V^{\phi^{\textrm{T}}_i}$ ($S^{\textrm{T}}$ is the state space for the target task, $V^{\phi^{\textrm{T}}_i}$ is the set of clock valuations for the clocks in $C^{\phi^{\textrm{T}}_i}$) is a finite set of extended states. 
For each episode, the index $i$ is first selected based on some heuristic criterion. For example, if the atomic predicates correspond to the regions to be reached in the state space, we select $i$ such that the centroid of the region corresponding to the atomic predicate in $\phi^{1,\textrm{T}}_{i}$ (as in $\phi_i^{\textrm{T}}=\phi^{1,\textrm{T}}_i\wedge\dots\wedge\phi^{n_i,\textrm{T}}_i$) has the nearest (Euclidean) distance from the initial state $s^{\textrm{T}}_0$. Then we perform RL in $X^{\textrm{T}}_i$. For Q-learning, after taking action $a^{\textrm{T}}$ at the current extended state $\chi^{\textrm{T}}_i=(s^{\textrm{T}}, q^{\textrm{T}}, v^{\textrm{T}})\in X^{\textrm{T}}_i$, a new extended state $\chi^{'\textrm{T}}_i=(s^{'\textrm{T}}, q^{'\textrm{T}}, v^{'\textrm{T}})\in X^{\textrm{T}}_i$ and a reward $R^{\textrm{T}}$ are obtained. We have the following update rule for the \textit{extended Q-function} values (denoted as $\bar{Q}$): 
\begin{align}\nonumber                                               
\bar{Q}(\chi_i^{\textrm{T}}, a^{\textrm{T}})\leftarrow (1-\alpha)\bar{Q}(\chi_i^{\textrm{T}}, a^{\textrm{T}})+\alpha(R^{\textrm{T}}+\gamma\max_{a^{'\textrm{T}}}\bar{Q}(\chi_i^{'\textrm{T}}, a^{'\textrm{T}})),                                                                     
\end{align} 
where $\alpha$ and $\gamma$ are the learning rate and discount factor, respectively.              

 \color{black} As in the introductory example, we construct a DTA [see Fig. \ref{TA} (b)] that accepts precisely the timed words that satisfy $\phi^{\textrm{T}}$ as there is only one sequential conjunctive subformula in $\phi^{\textrm{T}}$. We then perform RL in the extended state space $X^{\textrm{T}}=S^{\textrm{T}}\times\mathcal{Q}^{\phi^{\textrm{T}}}\times V^{\phi^{\textrm{T}}}$, where $S^{\textrm{T}}$ is the state space in the 9$\times$9 gridworld, $\mathcal{Q}^{\phi^{\textrm{T}}}=\{q^{\textrm{T}}_0, q^{\textrm{T}}_1, q^{\textrm{T}}_2, q^{\textrm{T}}_3\} $ and $V^{\phi^{\textrm{T}}}=\{0, 1, \dots, 40\}\times\{0, 1, \dots, 40\}$ (the set of clock valuations for the clocks $c^{\textrm{T}}_1$ and $c^{\textrm{T}}_2$). \color{black}\\    
 	                      
\begin{figure}
	\centering
	\begin{subfigure}[b]{0.6\textwidth}
		\centering
		\includegraphics[width=\textwidth]{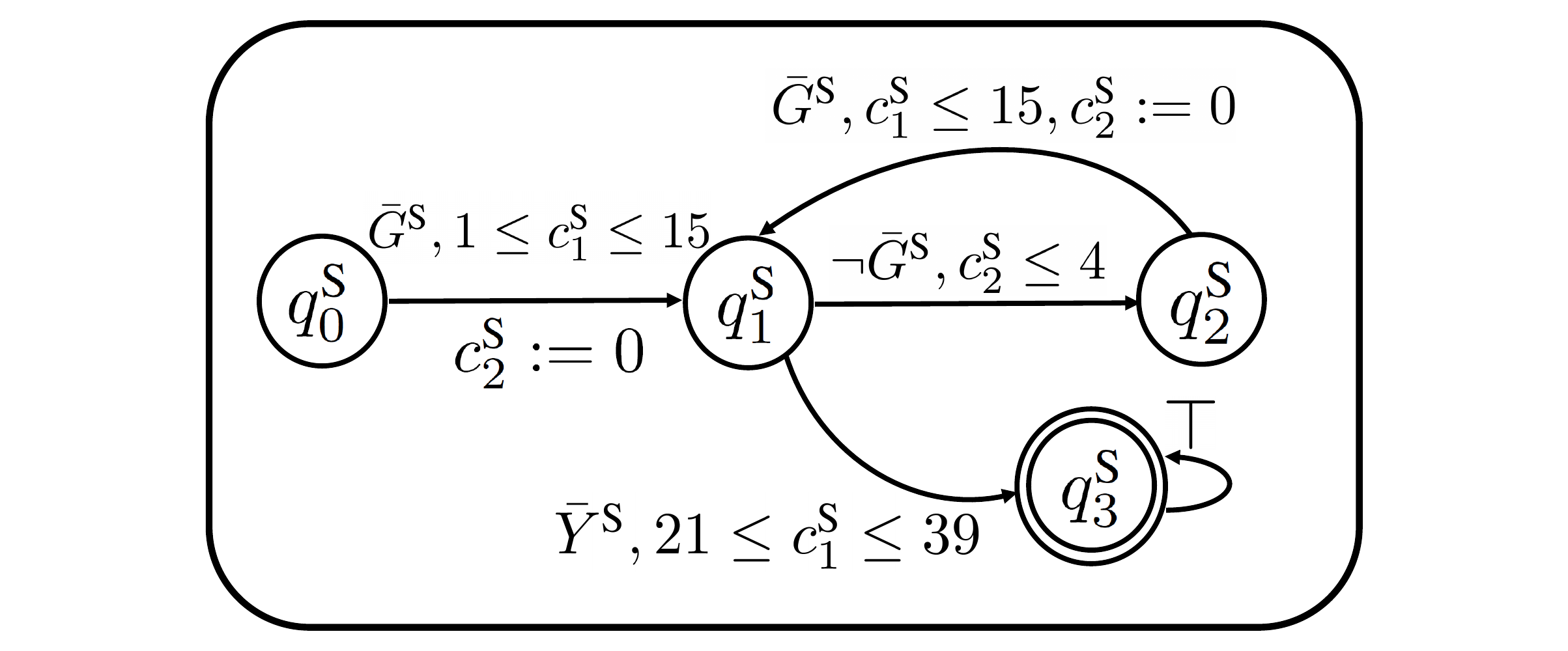}
		\caption{}
	\end{subfigure}
	\hfill
	\begin{subfigure}[b]{0.6\textwidth}
		\centering
		\includegraphics[width=\textwidth]{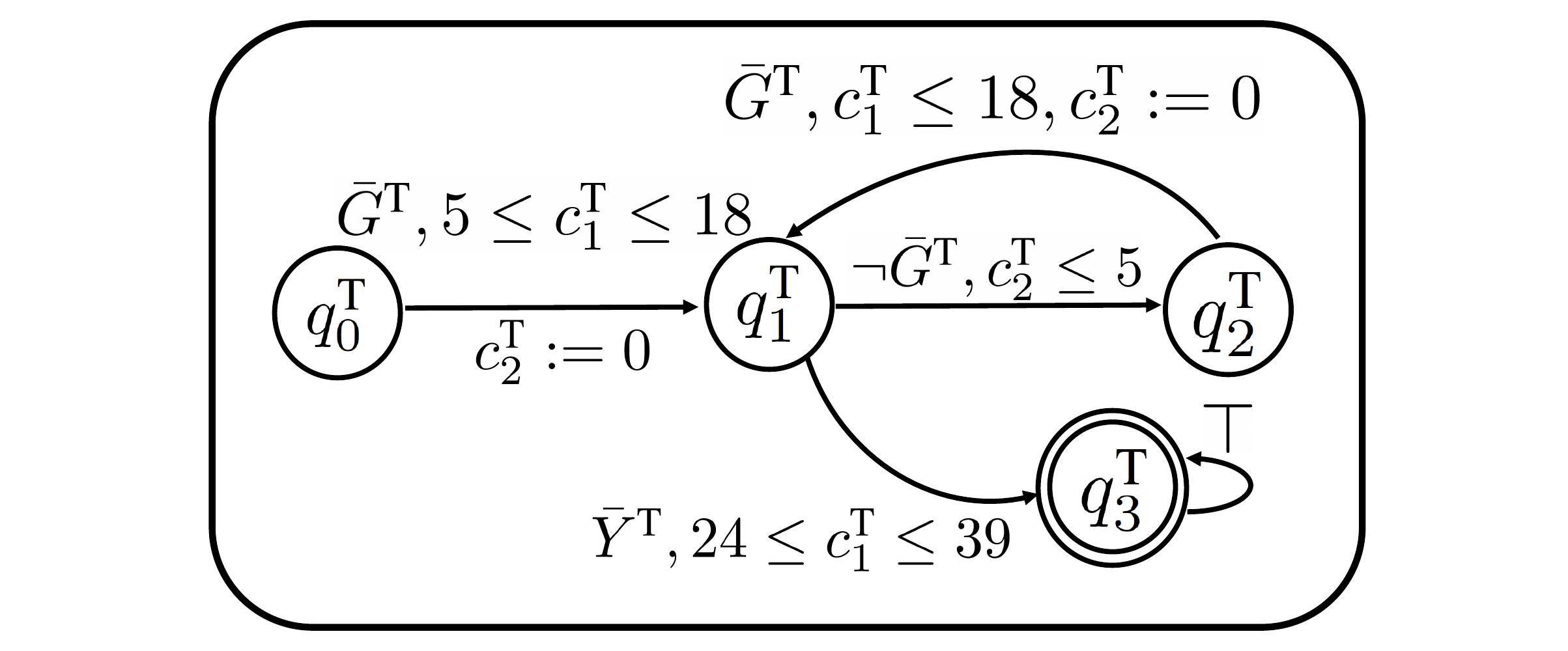}
		\caption{}
	\end{subfigure}
	\caption{\color{black}The deterministic timed automata (DTA) of two structurally equivalent formulas (a) $\Diamond_{[1,15]}\Box_{[0,4]}\bar{G}^{\textrm{S}}\wedge\Diamond_{[21,39]}\bar{Y}^{\textrm{S}}$ and (b)
		$\Diamond_{[5,18]}\Box_{[0,5]}\bar{G}^{\textrm{T}}\wedge \Diamond_{[24,39]}\bar{Y}^{\textrm{T}}$. The locations $q^{\textrm{S}}_0$, $q^{\textrm{S}}_1$, $q^{\textrm{S}}_2$ and $q^{\textrm{S}}_3$ correspond to $q^{\textrm{T}}_0$, $q^{\textrm{T}}_1$, $q^{\textrm{T}}_2$ and $q^{\textrm{T}}_3$, respectively. The atomic predicate $\bar{G}^{\textrm{S}}$ and $\bar{Y}^{\textrm{S}}$ correspond to $\bar{G}^{\textrm{T}}$ and $\bar{Y}^{\textrm{T}}$, respectively. The clocks $c^{\textrm{S}}_1$ and $c^{\textrm{S}}_2$ correspond to $c^{\textrm{T}}_1$ and $c^{\textrm{T}}_2$, respectively.}  
	\label{TA}
\end{figure}
	
	\subsection{Transfer of Extended Q-functions Based on Identified Logical Transferability}                    
	\label{Sec_Q}

Next, we introduce the transfer of extended Q-functions if logical transferability can be identified from Section \ref{Sec_transferLogic}. 

We assume that the sets of actions in the source task and the target task are the same, denoted as $A$. For the satisfying formula $\phi^{\textrm{S}}=\phi^{\textrm{S}}_1\vee\dots\vee\phi^{\textrm{S}}_{m}$ in the SDNF, we construct a DTA corresponding to each $\phi_i^{\textrm{S}}$ and perform Q-learning in the extended state space for the source task. We denote the obtained optimal extended Q-functions as $\bar{Q}^{\textrm{S}\ast}(s^{\textrm{S}}, q^{\textrm{S}}, v^{\textrm{S}},a)$. In the following, we explain the details for transferring $\bar{Q}^{\textrm{S}\ast}(s^{\textrm{S}}, q^{\textrm{S}}, v^{\textrm{S}},a)$ to the target task based on the identified logical transferability. 
	
From Definitions \ref{MITL_equ} and \ref{MITL_transfer}, if $\phi^{\textrm{T}}=\phi^{\textrm{T}}_1\vee\dots\vee\phi^{\textrm{T}}_{m^\textrm{T}}$ is structurally transferable from $\phi^{\textrm{S}}=\phi^{\textrm{S}}_1\vee\dots\vee\phi^{\textrm{S}}_{m^\textrm{S}}$, \color{black}then for all index $i\in[1,m^\textrm{T}]$, the sequential conjunctive subformulas $\phi^{\textrm{T}}_i$ and $\phi^{\textrm{S}}_{\hat{i}}$ are structurally equivalent, where $\hat{i}=i\mod m^{\textrm{S}}$ (where $\mod$ denotes the \textit{modulo} operation). \color{black}For the DTA $\mathcal{A}^{\phi^{\textrm{T}}_i}$ and $\mathcal{A}^{\phi^{\textrm{S}}_{\hat{i}}}$ constructed from $\phi^{\textrm{T}}_i$ and $\phi^{\textrm{S}}_{\hat{i}}$ respectively, it can be proven that we can establish bijective mappings: $\xi^{i,\hat{i}}_{\Sigma}: 2^{AP}\rightarrow2^{AP}$, $\xi^{i,\hat{i}}_{\mathcal{Q}}: \mathcal{Q}^{\phi_i^{\textrm{T}}}\rightarrow\mathcal{Q}^{\phi_{\hat{i}}^{\textrm{S}}}$ and $\xi^{i,\hat{i}}_{C}: C^{\phi_i^{\textrm{T}}}\rightarrow C^{\phi_{\hat{i}}^{\textrm{S}}}$ such that the \textit{structures} of $\mathcal{A}^{\phi_i^{\textrm{T}}}$ and $\mathcal{A}^{\phi_{\hat{i}}^{\textrm{S}}}$ are preserved under these bijective mappings \cite{Glushkov1961}. Specifically, we have $\xi^{i,\hat{i}}_{\mathcal{Q}}(q^{0\phi_{i}^{\textrm{T}}})=q^{0\phi_{\hat{i}}^{\textrm{S}}}$, $\xi^{i,\hat{i}}_{\mathcal{Q}}[\mathcal{F}^{\phi_{i}^{\textrm{T}}}]=\mathcal{F}^{\phi_{\hat{i}}^{\textrm{S}}}$ (where $\xi^{i,\hat{i}}_{\mathcal{Q}}[\mathcal{F}^{\phi_{i}^{\textrm{T}}}]$ denotes the point-wise application of $\xi^{i,\hat{i}}_{\mathcal{Q}}$ to elements of $\mathcal{F}^{\phi_{i}^{\textrm{T}}}$). Besides, for any $\rho\in2^{AP}$ and any $q, q'\in\mathcal{Q}^{\phi_{i}^{\textrm{T}}}$, we have that                                  
	\begin{align}\nonumber                                                
	e^{\phi_{i}^{\textrm{T}}}=(q, \rho, q', \varphi^{\phi_{i}^{\textrm{T}}}_{C_1}, r^{\phi_{i}^{\textrm{T}}}_{C_1})\in \Delta^{\phi_{i}^{\textrm{T}}}                                                                                
	\end{align}  
	holds if and only if                                                                                   
	\begin{align}\nonumber                                               
	e^{\phi_{\hat{i}}^{\textrm{S}}}=(\xi^{i,\hat{i}}_{\mathcal{Q}}(q), \xi^{i,\hat{i}}_{\Sigma}(\rho), \xi^{i,\hat{i}}_{\mathcal{Q}}(q'), \varphi^{\phi_{\hat{i}}^{\textrm{S}}}_{C_2},                         r^{\phi_{\hat{i}}^{\textrm{S}}}_{C_2})\in\Delta^{\phi_{\hat{i}}^{\textrm{S}}}                            
	\end{align} 
	holds, where $C_1=C^{\phi_{i}^{\textrm{T}}}$ and $C_2=\xi^{i,\hat{i}}_{C}[C^{\phi_{i}^{\textrm{T}}}]$. See Fig. \ref{TA} for an illustrative example.

	\begin{algorithm}[t]
	\caption{Transfer of Extended Q-functions Based on Logical Transferability.}                                                                                                                                                                    
	\label{policy}
	\begin{algorithmic}[1]
		\State \color{black}\textbf{Inputs:}  $\phi^{\textrm{S}}$, $\phi^{\textrm{T}}$, $X^{\textrm{S}}$, $X^{\textrm{T}}$, $A$\color{black}
		\For{all index $i$ and all $\chi_i^{\textrm{T}}=(s^{\textrm{T}},q^{\textrm{T}}, v^{\textrm{T}})\in X^{\textrm{T}}_i$} 
		\State Identify the unique formula $\phi^{j,\textrm{T}}_i$ and its atomic predicate $\rho^{j,\textrm{T}}_{i}$ \label{identify_formula}
		\State $\hat{i}\gets\mod(i,m^{\textrm{S}})$
		\State $\rho^{\textrm{S}}\gets\xi^{i,\hat{i}}_{\Sigma}(\rho^{j,\textrm{T}}_{i})$ \label{rho_map}
		\State $s^{\textrm{S}}\gets\argmin\limits_{s\in S^{\textrm{S}}}\norm{\big(s-Cr(\rho^{\textrm{S}})\big)-\big(s^{\textrm{T}}-Cr(\rho^{j,\textrm{T}}_{i})\big)}$ \label{s_map}
		\State $q^{\textrm{S}}\gets\xi^{i,\hat{i}}_{\mathcal{Q}}(q^{\textrm{T}})$ \label{q_map}
		\State \textbf{for} all $c_l^{\textrm{T}}\in C^{\phi_i^{\textrm{T}}}$ and all $v_l^{\textrm{T}}\in V_l^{\phi_i^{\textrm{T}}}$  \textbf{do}               
		\State $c_k^{\textrm{S}}\gets\xi^{i,\hat{i}}_{C}(c_l^{\textrm{T}})$  \label{c_map_s}                            
		\State $v_k^{\textrm{S}}\gets\argmin\limits_{v_k\in V_k^{\phi_{\hat{i}}^{\textrm{S}}}}\norm{v_k-v_l^{\textrm{T}}}$  \label{c_map_e}                                      
		\State \textbf{end for} 
		\State $\bar{Q}^{\textrm{T}}(s^{\textrm{T}}, q^{\textrm{T}}, v^{\textrm{T}}, a)\gets\bar{Q}^{\textrm{S}\ast}(s^{\textrm{S}}, q^{\textrm{S}}, v^{\textrm{S}}, a)$ \textbf{for} all $a\in A$                            
		\EndFor
		\State Return $\bar{Q}^{\textrm{T}}$  
	\end{algorithmic} 
\end{algorithm}                                                          
	
    Algorithm \ref{policy} is for the transfer of the extended Q-functions.  For all indices $i$, we first identify a unique primitive MITL$_{f}$ formula $\phi^{j,\textrm{T}}_i$ (as in $\phi_i^{\textrm{T}}=\phi^{1,\textrm{T}}_i\wedge\dots\wedge\phi^{n_i,\textrm{T}}_i$) that is to be satisfied at each extended state $\chi_i^{\textrm{T}}=(s^{\textrm{T}},q^{\textrm{T}}, v^{\textrm{T}})$ (Line \ref{identify_formula}). Specifically, according to $q^{\textrm{T}}$ and $v^{\textrm{T}}$, we identify the index $j$ such that $\phi^{1,\textrm{T}}_{i}$, $\dots$, $\phi^{(j-1),\textrm{T}}_{i}$ are already satisfied while $\phi^{j,\textrm{T}}_{i}$ is still not satisfied.

    Next, we identify the state, location and clock valuation in the extended state $\chi_{\hat{i}}^{\textrm{S}}$ that are the most \textit{similar} to $s^{\textrm{T}}$, $q^{\textrm{T}}$ and $v^{\textrm{T}}$ respectively in the extended state $\chi_i^{\textrm{T}}$. 
    
	\noindent\textbf{Identification of State:} We first identify the atomic predicate $\rho^{j,\textrm{T}}_{i}$ in the primitive MITL$_{f}$ formula $\phi^{j,\textrm{T}}_i$. Then we identify the atomic predicate $\rho^{\textrm{S}}$ corresponding to $\rho^{j,\textrm{T}}_{i}$ through the mapping $\xi^{i,\hat{i}}_{\Sigma}$ (Line \ref{rho_map}). We use $Cr(\rho)$ to denote the centroid of the region corresponding to the atomic predicate $\rho$ and $\norm{\cdot}$ to denote the 2-norm. We identify the state $s^{\textrm{S}}$ in $S^{\textrm{S}}$ such that the relative position of $s^{\textrm{S}}$ with respect to $Cr(\rho^{\textrm{S}})$ is the most similar (measured in Euclidean distance) to the relative position of $s^{\textrm{T}}$ with respect to $Cr(\rho^{j,\textrm{T}}_{i})$ (Line \ref{s_map}). \color{black} As in the introductory example, at the locations $q^{\textrm{S}}_0$ and $q^{\textrm{T}}_0$, we first identify the atomic predicate $\bar{G}^{\textrm{T}}$, then the mapping $\xi^{1,1}_{\Sigma}$ maps $\bar{G}^{\textrm{T}}$ to its corresponding atomic predicate $\bar{G}^{\textrm{S}}$. For a state $s^{\textrm{T}}$ in $B^{\textrm{T}}$ in the target task, we identify the state $s^{\textrm{S}}$ in $B^{\textrm{S}}$ (see Fig. \ref{map_intro_infer}) in the source task, as the relative position of $s^{\textrm{S}}$ in $B^{\textrm{S}}$ with respect to the centroid of $\bar{G}^{\textrm{S}}$ is the most similar (measured in Euclidean distance) to the relative position of $s^{\textrm{T}}$ in $B^{\textrm{T}}$ with respect to the centroid of $\bar{G}^{\textrm{T}}$.   \color{black}                                                                                                                      
	
	\noindent\textbf{Identification of Location:} We identify the location $q^{\textrm{S}}$ corresponding to $q^{\textrm{T}}$ through the mapping $\xi^{i,\hat{i}}_{\mathcal{Q}}$ (Line \ref{q_map}). 
	\color{black} As in the introductory example, the mapping $\xi^{1,1}_{\mathcal{Q}}$ maps the locations $q^{\textrm{T}}_0$, $q^{\textrm{T}}_1$, $q^{\textrm{T}}_2$ and $q^{\textrm{T}}_3$ to the locations $q^{\textrm{S}}_0$, $q^{\textrm{S}}_1$, $q^{\textrm{S}}_2$ and $q^{\textrm{S}}_3$, respectively (see Fig. \ref{TA}).  \color{black}      
	
	\noindent\textbf{Identification of Clock Valuation:} For each clock $c_l^{\textrm{T}}\in C^{\phi_i^{\textrm{T}}}$, we identify the clock $c_k^{\textrm{S}}$ corresponding to $c_l^{\textrm{T}}$ through the mapping $\xi^{i,\hat{i}}_{C}$ (Line \ref{c_map_s}). \color{black} As in the introductory example, the mapping $\xi^{1, 1}_{C}$ maps the clocks $c^{\textrm{T}}_1$ and $c^{\textrm{T}}_2$ to the corresponding clocks $c^{\textrm{S}}_1$ and $c^{\textrm{S}}_2$, respectively (see Fig. \ref{TA}). \color{black} Then for each clock valuation $v_l^{\textrm{T}}\in V_l^{\phi_i^{\textrm{T}}}$, we identify the (scalar) clock valuation $v^{\textrm{S}}_k\in V_k^{\phi_{\hat{i}}^{\textrm{S}}}$ which is the most similar (in scalar value) to $v_l^{\textrm{T}}$ (Line \ref{c_map_e}).

    In this way, $\bar{Q}^{\textrm{S}\ast}(s^{\textrm{S}}, q^{\textrm{S}}, v^{\textrm{S}},a)$ from the source task are transferred to $\bar{Q}^{\textrm{T}}(s^{\textrm{T}}, q^{\textrm{T}}, v^{\textrm{T}},a)$ in the target task. Finally, we perform Q-learning of $\mathcal{T}^{\textrm{T}}$ in the extended state space, starting with the transferred extended Q-functions $\bar{Q}^{\textrm{T}}(s^{\textrm{T}}, q^{\textrm{T}}, v^{\textrm{T}},a)$.

\section{Case Studies}                    
In this section, we illustrate the proposed approach on \color{black}two case studies. In Case Study 1, the gridworlds in the source environment and the target environment are of the same size. In Case Study 2, the gridworld in the target environment is larger than that in the source environment.  \color{black}   

\subsection{Case Study 1}                     
\label{case_I}
In Case Study 1, we consider the introductory
example in the 9$\times$9 gridworld as shown in Fig.\ref{map_intro}. The robot has three possible actions at each time step: go straight, turn left or turn right. After going straight, the robot may slip to adjacent cells with probability of 0.04. After turning left or turning right, the robot may stay in the original direction with probability of 0.03. We first perform Q-learning on the $\tau$-states (i.e., the $\tau$-horizon trajectory involving the current state and the most recent $\tau-1$ past states, see \cite{Aksaray2016}, we set $\tau$=5) for the source task and the target task. We set $\alpha=0.8$ and $\gamma=0.99$. For each episode, the initial state is randomly selected.

We use the first 10000 episodes of Q-learning as the \textit{data collection phase}. From the source task, 46 out of the 10000 trajectories with cumulative rewards above 0 are labeled as 1, and 200 trajectories randomly selected out of the remaining 9954 trajectories are labeled as -1. From the target task, 19 trajectories are labeled as 1 and 200 trajectories are labeled as -1 with the same labeling criterion. 

For the inference problem (Problem \ref{problem1}), we set $\varrho_{\rm{th}}=4$ and $\zeta=0.95$. For Algorithm \ref{algTree}, we set $\lambda=0.01$ and $h_{\textrm{max}}=2$. We use the position of the robot as the state, and the atomic predicates $\rho$ correspond to the rectangular regions in the 9$\times$9 gridworld. For computing the information gain of MITL$_{f}$ formulas, we use the uniform distribution for the prior probability distribution $\mathcal{G}_L$. Following the first two steps illustrated in Section \ref{Sec_transferLogic}, we obtain the following satisfying formulas:                                                                                                               
\[		                 
\begin{split}
\phi^{\textrm{S}}=&\Diamond_{[1,15]}\Box_{[0,4]}\bar{G}^{\textrm{S}}\wedge \Diamond_{[21,39]}\bar{Y}^{\textrm{S}}~\textrm{and}\\  
\phi^{\textrm{T}}=&\Diamond_{[5,18]}\Box_{[0,5]}\bar{G}^{\textrm{T}}\wedge \Diamond_{[24,39]}\bar{Y}^{\textrm{T}},
\end{split} 
\]
where 
\color{black}
\[ 
\begin{split}
&\bar{G}^{\textrm{S}}=(x>=3)\wedge(x<=4)\wedge(y>=3)\wedge(y<=5),\\
&\bar{Y}^{\textrm{S}}=(x>=7)\wedge(x<=8)\wedge(y>=5)\wedge(y<=8),\\
&\bar{G}^{\textrm{T}}=(x>=5)\wedge(x<=6)\wedge(y>=6)\wedge(y<=7),\\
&\bar{Y}^{\textrm{T}}=(x>=4)\wedge(x<=7)\wedge(y>=1)\wedge(y<=2).
\end{split}                                                                         
\]
$\phi^{\textrm{S}}$ reads as ``first reach $\bar{G}^{\textrm{S}}$ during the time interval [1, 15] and stay there for 4 time units, then reach $\bar{Y}^{\textrm{S}}$ during the time interval [21, 39]''. $\phi^{\textrm{T}}$ reads as ``first reach $\bar{G}^{\textrm{T}}$ during the time interval [5, 18] and stay there for 5 time units, then reach $\bar{Y}^{\textrm{T}}$ during the time interval [24, 39]''. \color{black}The regions $\bar{G}^{\textrm{S}}$, $\bar{Y}^{\textrm{S}}$, $\bar{G}^{\textrm{T}}$ and $\bar{Y}^{\textrm{T}}$ are shown in Fig. \ref{map_I}  (a) (b). It can be seen that $\phi^{\textrm{T}}$ is structurally equivalent with $\phi^{\textrm{S}}$, hence logical transferability is identified.

For comparison, we also obtain $\phi^{'\textrm{T}}$ without considering the information gain, i.e., by setting $\lambda=0$ in (\ref{obj}):
\[
\begin{split}
\phi^{'\textrm{T}}=\Diamond_{[1,18]}\Box_{[0,4]}\bar{G}^{\textrm{T}}\wedge \Diamond_{[24,40]}\bar{Y}^{'\textrm{T}},
\end{split}
\]            
where 
\color{black}
\[ 
\begin{split}
&\bar{G}^{\textrm{T}}=(x>=5)\wedge(x<=6)\wedge(y>=6)\wedge(y<=7),\\
&\bar{Y}^{'\textrm{T}}=(x>=0)\wedge(x<=8)\wedge(y>=0)\wedge(y<=2).
\end{split}                                                                         
\]                                  
$\phi^{'\textrm{T}}$ reads as ``first reach $\bar{G}^{\textrm{T}}$ during the time interval [1, 18] and stay there for 4 time units, then reach $\bar{Y}^{'\textrm{T}}$ during the time interval [24, 40]''.\color{black} $\bar{G}^{\textrm{T}}$ and $\bar{Y}^{'\textrm{T}}$ are shown in Fig. \ref{map_I} (c). It can be seen that $\phi^{\textrm{T}}$ implies $\phi^{'\textrm{T}}$, hence $\phi^{'\textrm{T}}$ is less informative than $\phi^{\textrm{T}}$ with respect to the prior probability distribution $\mathcal{G}_L$. 

\begin{figure}[th]
	\centering
	\includegraphics[width=12cm]{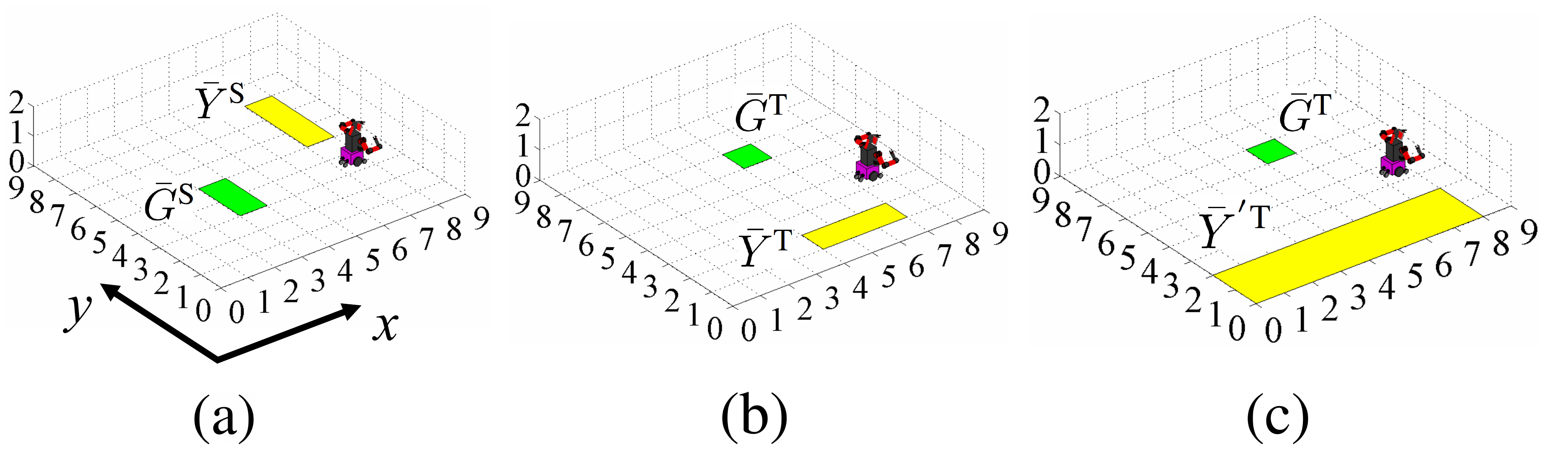}\caption{Inferred regions in Case Study 1.}                   
	\label{map_I}                                                                                                      
\end{figure}

We use Method I to refer to the Q-learning on the $\tau$-states. In comparison with Method I, we perform Q-learning in the extended state space with the following three methods:\\
Method II: Q-learning with $\phi^{'\textrm{T}}$ (i.e., on the extended state that includes the locations and clock valuations of the timed automata constructed from $\phi^{'\textrm{T}}$). \\
Method III: Q-learning with $\phi^{\textrm{T}}$. \\
Method IV: Q-learning with $\phi^{\textrm{T}}$ and starting from the transferred extended Q-functions.                                       

Fig. \ref{result_I} shows the learning results with the four different methods. Method I takes an average of 834590 episodes to converge to the optimal policy (with the first 50000 episodes shown in Fig. \ref{result_I}), while Method III and Method IV take an average of 13850 episodes and 2220 episodes for convergence to the optimal policy, respectively. It should be noted that although Method II performs better than Method I in the first 50000 episodes, it does not achieve optimal performance in 2 million episodes (as $\phi^{'\textrm{T}}$ is not sufficiently informative). In sum, the sampling efficiency for the target task is improved by up to one order of magnitude by performing RL in the extended state space with the inferred formula $\phi^{\textrm{T}}$, and further improved by up to another order of magnitude using the transferred extended Q-functions.

\begin{figure}[th]
	\centering
	\includegraphics[width=10cm]{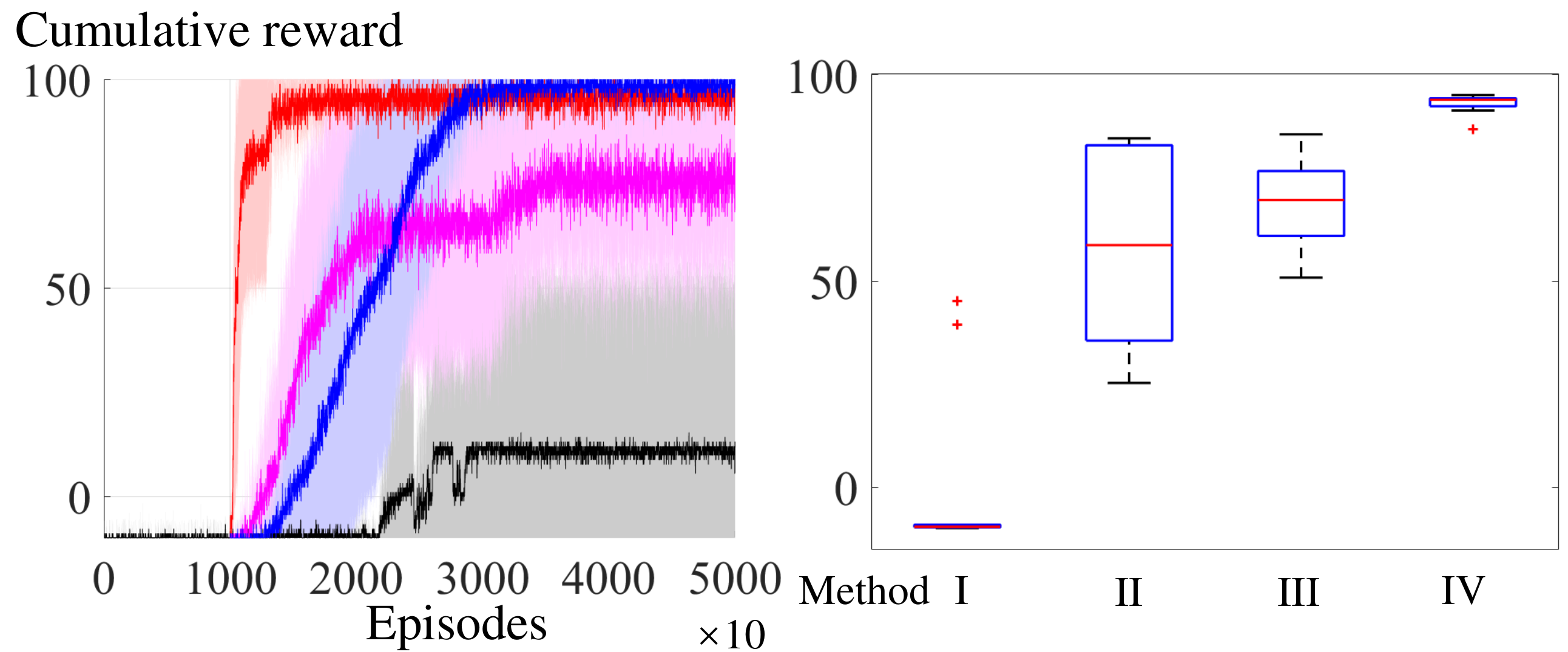}\caption{Learning results in Case Study 1: cumulative rewards of 10 independent simulation runs averaged for every 10 episodes (left) and boxplot of the 10 runs for the average cumulative rewards of 40000 episodes after the data collection phase (right). Black: Method I; magenta: Method II; blue: Method III; red: Method IV.}
	\label{result_I}
\end{figure}                                                                                       

\color{black}
\subsection{Case Study 2}  

\begin{figure}[th]
	\centering
	\color{black}
	\includegraphics[width=10cm]{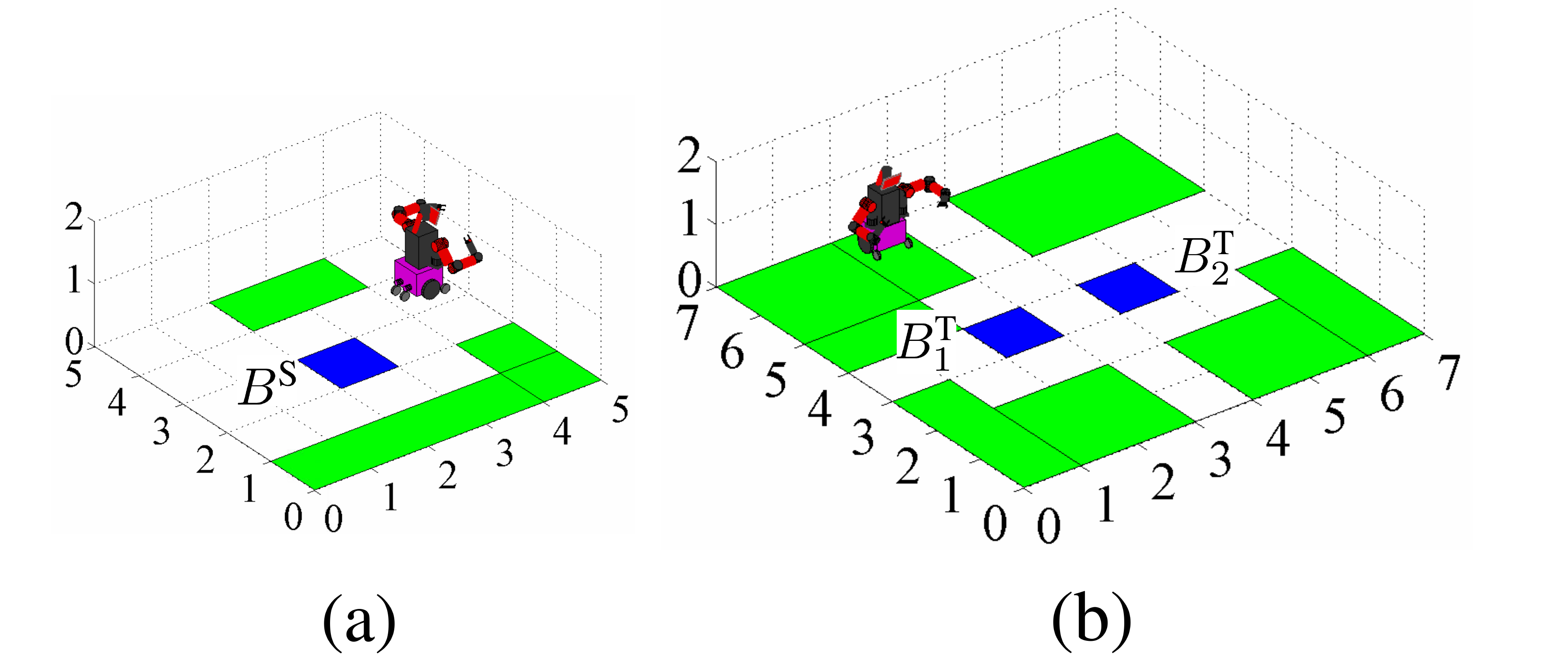}\caption{The source environment (a) and target environment (b) in Case Study 2.}
	\label{map_II}
\end{figure}                                                                

In Case Study 2, we consider an example where the gridworlds in the source environment and target environment are of different sizes. In the source environment [a $5\times5$ gridworld as shown in Fig. \ref{map_II} (a)], the robot obtains a reward of 100 for each time unit (within 25 time units) it is in the two green regions. In the meantime, there is an underlying rule for the source task requiring that,  within 25 time units, the robot should come back to a blue region ($B^{\textrm{S}}$) in every 8 time units. If the robot breaks the rule, it will fail the task immediately and obtain a reward of -800. In the target environment [a $7\times7$ gridworld as shown in Fig. \ref{map_II} (b)], there are four green regions and two blue regions ($B^{\textrm{T}}_1$ and $B^{\textrm{T}}_2$). The underlying rule for the target task requires that, within 40 time units, the robot should come back to one of the blue regions in every 10 time units (it should be the same blue region every time), with the same reward of -800 for breaking the rule. In both tasks, the blue regions are a priori unknown to the robot. The robot has three possible actions at each time step: go straight, turn left or turn right. After going straight, the robot may slip to adjacent cells with probability of 0.04. After turning left or turning right, the robot may stay in the original direction with probability of 0.03. 

We set $\alpha=0.8$ and $\gamma=0.99$. For each episode, the initial state is randomly selected. We use the first 1000 episodes of Q-learning on the $\tau$-states as the \textit{data collection phase}. We label the collected trajectories based on the lengths of the trajectories as we intend to infer an MITL$_{f}$ formula that enables the robot to obey the underlying rule for longer time (which is essential to gain higher rewards in the long run). From the source task, 22 out of the 1000 trajectories with lengths of at least 20 are labeled 1, and 200 trajectories randomly selected out of the remaining 978 trajectories are labeled -1. From the  target task, 34 trajectories are labeled 1 and 200 trajectories are labeled as -1 with the same labeling criterion. We delete the states at the last time unit of the trajectories with label 1 so that these trajectories all represent behaviors that obey the rule until the end of the time. 

Following the first two steps illustrated in Section \ref{Sec_transferLogic} and using the same hyperparameters as in Case Study 1, we obtain the following satisfying formulas:                                                                                                               
\[ 
\begin{split}
\psi^{\textrm{S}}=&\Box_{[0,25]}\Diamond_{[0,8]}\bar{B}^{\textrm{S}},\\                          
\psi^{\textrm{T}}=&\Box_{[0,40]}\Diamond_{[0,10]}\bar{B}^{\textrm{T}}_1\vee\Box_{[0,40]}\Diamond_{[0,10]}\bar{B}^{\textrm{T}}_2,
\end{split}                                                                         
\]
where
\[ 
\begin{split}
 &\bar{B}^{\textrm{S}}=(x>=2)\wedge(x<=3)\wedge(y>=2)\wedge(y<=3),\\
 &\bar{B}^{\textrm{T}}_1=(x>=2)\wedge(x<=3)\wedge(y>=3)\wedge(y<=4),\\
 &\bar{B}^{\textrm{T}}_2=(x>=4)\wedge(x<=5)\wedge(y>=3)\wedge(y<=4).
\end{split}                                                                         
\]
$\psi^{\textrm{S}}$ reads as ``during the time interval [0, 25], reach $\bar{B}^{\textrm{S}}$ for at least 1 time unit in every 8 time units''. $\psi^{\textrm{T}}$ reads as ``during the time interval [0, 40], either reach $\bar{B}^{\textrm{T}}_1$ for at least 1 time unit in every 10 time units, or reach $\bar{B}^{\textrm{T}}_2$ for at least 1 time unit in every 10 time units''. As shown in Fig. \ref{map_II_infer} (a) and (b), the inferred regions $\bar{B}^{\textrm{S}}$, $\bar{B}^{\textrm{T}}_1$ and $\bar{B}^{\textrm{T}}_2$ are the same as $B^{\textrm{S}}$, $B^{\textrm{T}}_1$ and $B^{\textrm{T}}_2$ (in Fig. \ref{map_II}), respectively. It can be seen that $\psi^{\textrm{T}}$ is structurally transferable from $\psi^{\textrm{S}}$, hence logical transferability is identified.                                                        

\begin{figure}[th]
	\centering
	\color{black}
	\includegraphics[width=12cm]{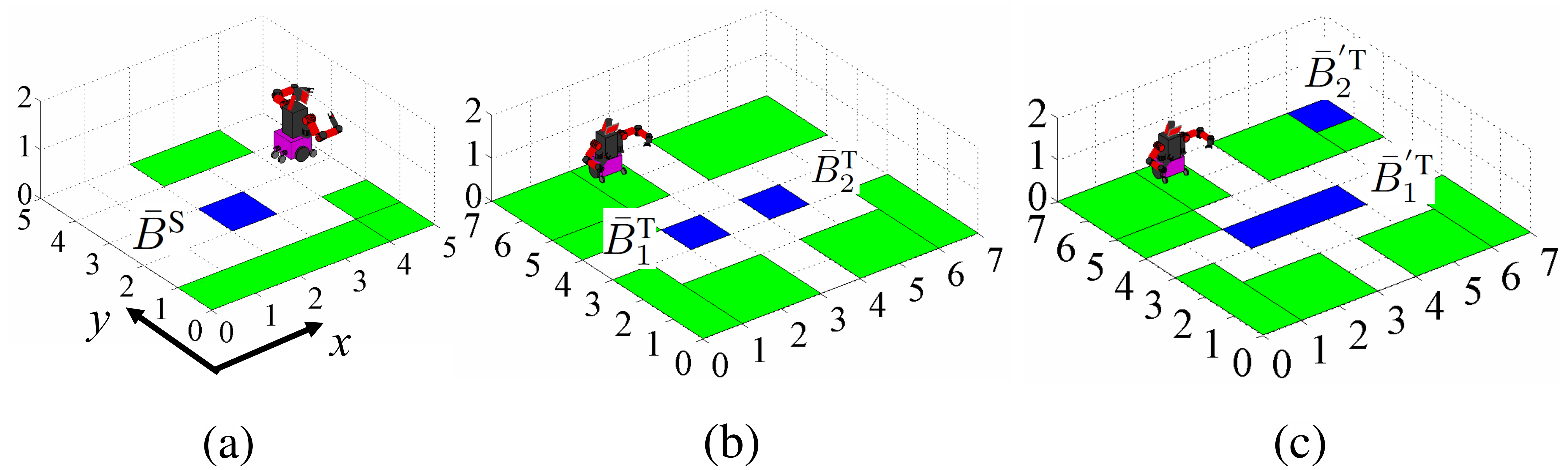}\caption{Inferred regions in Case Study 2.}
	\label{map_II_infer}
\end{figure}

\begin{figure}[th]
	\centering
	\color{black}
	\includegraphics[width=12cm]{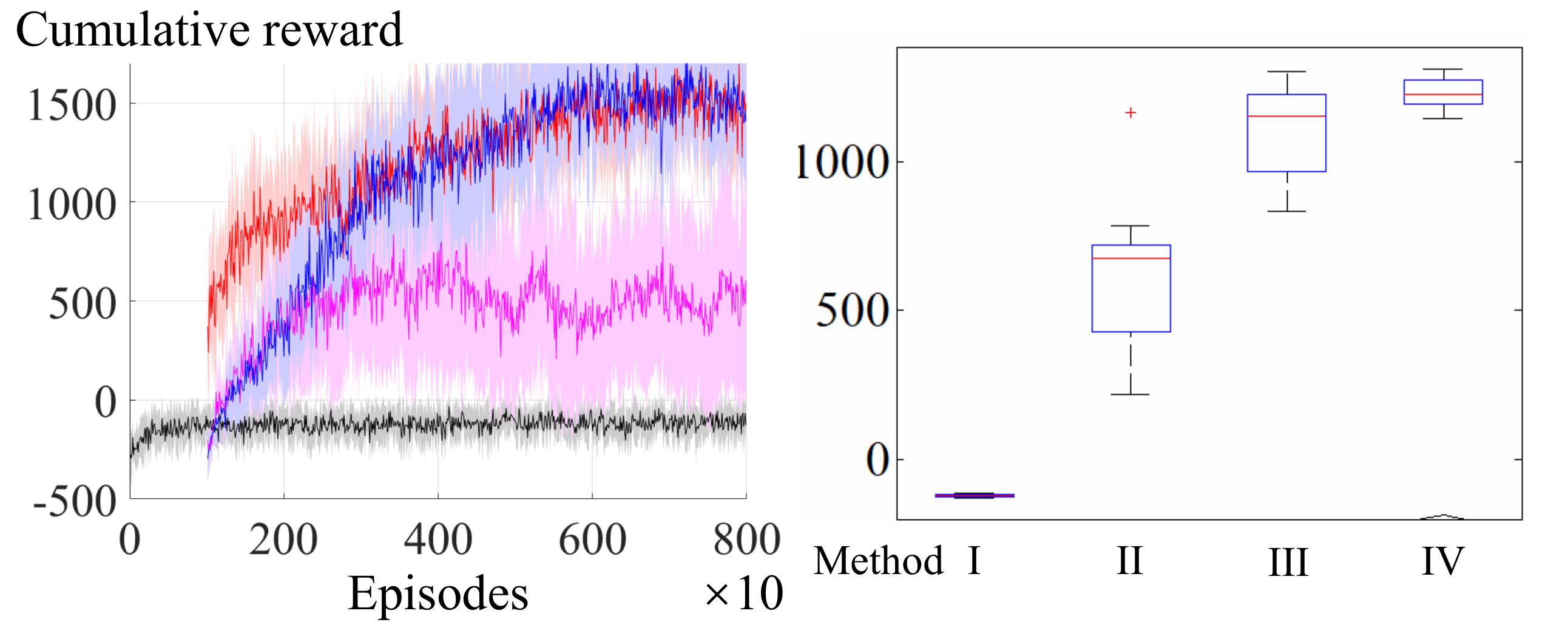}\caption{\color{black}Learning results in Case Study 2: cumulative rewards of 10 independent simulation runs averaged for every 10 episodes (left) and boxplot of the 10 runs for the average cumulative rewards of 7000 episodes after the data collection phase (right). Black: method I; magenta: method II; blue: method III; red: method IV.}
	\label{result_II}
\end{figure}    

For comparison, we also obtain $\psi^{'\textrm{T}}$ without considering the information gain, i.e., by setting $\lambda=0$ in (\ref{obj}):
\[
\begin{split}
\psi^{'\textrm{T}}=\Box_{[0,40]}\Diamond_{[0,10]}\bar{B}^{'\textrm{T}}_1\vee\Box_{[0,40]}\Diamond_{[0,9]}\bar{B}^{'\textrm{T}}_2,
\end{split} 
\]
where
\[ 
\begin{split}
&\bar{B}^{'\textrm{T}}_1=(x>=2)\wedge(x<=5)\wedge(y>=3)\wedge(y<=4),\\   
&\bar{B}^{'\textrm{T}}_2=(x>=6)\wedge(x<=7)\wedge(y>=6)\wedge(y<=7).
\end{split}                                                                         
\]
$\psi^{'\textrm{T}}$ reads as ``during the time interval [0, 40], either reach $\bar{B}^{'\textrm{T}}_1$ for at least 1 time unit in every 10 time units, or reach $\bar{B}^{'\textrm{T}}_2$ for at least 1 time unit in every 9 time units''. The regions $\bar{B}^{'\textrm{T}}_1$ and $\bar{B}^{'\textrm{T}}_2$ are as shown in Fig. \ref{map_II_infer} (c). It can be seen that $\psi^{\textrm{T}}$ implies $\psi^{'\textrm{T}}$, hence $\psi^{'\textrm{T}}$ is less informative than $\psi^{\textrm{T}}$ with respect to the uniform prior probability distribution. 

Similar to Case Study 1, we perform Q-learning with methods I, II, III and IV, where we use the formula $\psi^{'\textrm{T}}$ for method II, and $\psi^{\textrm{T}}$ for method III and method IV. Fig. \ref{result_II} shows the learning results with the four different methods.  Convergence to the optimal policy is not achieved in 2 million episodes by method I and method II, while method III and method IV take an average of 1540 episodes and 610 episodes respectively for exceeding cumulative rewards of 1000 for the first time, and both take an average of about 6000 episodes for convergence to the optimal policy.      
\color{black}              

\section{Discussions}                                                                    
We proposed a transfer learning approach for temporal tasks based on logical transferability. We have shown the improvement of sampling efficiency in the target task using the proposed method. 

There are several limitations of the current approach, which leads to possible directions for future work. Firstly, the proposed logical transferability is a qualitative measure of the logical similarities between the source task and the target task. Quantitative measures of logical similarities can be further established using similarity metrics between the inferred temporal logic formulas from the two tasks. \color{black}Secondly, as some information about the task may not be discovered during the initial episodes of reinforcement learning (especially for more complicated tasks), the inferred temporal logic formulas can be incomplete or biased. We will develop methods for more complicated tasks by either breaking the tasks into simpler subtasks, or iteratively performing inference of temporal logic formulas and reinforcement learning as a closed loop process. Finally, \color{black}we use Q-learning as the underlying learning algorithm for the transfer learning approach. The same methodology can be also applied to other forms of reinforcement learning, such as actor-critic methods or model-based reinforcement learning.

\section*{Acknowledgements}    
This research was partially supported by AFOSR FA9550-19-1-0005, DARPA D19AP00004, NSF 1652113, ONR N00014-18-1-2829 and NASA 80NSSC19K0209. 

\bibliographystyle{IEEEtran}
\bibliography{zheRLrefwhole}      
	
\end{document}